\documentclass{article}

\usepackage[preprint]{neurips_2024}

\usepackage[utf8]{inputenc} 
\usepackage[T1]{fontenc}    
\usepackage{hyperref}       
\usepackage{url}            
\usepackage{booktabs}       
\usepackage{amsfonts}       
\usepackage{nicefrac}       
\usepackage{microtype}      
\usepackage{xcolor}         

\usepackage{microtype}
\usepackage{graphicx}
\usepackage{subfigure}
\usepackage{booktabs} 

\usepackage{amsmath}
\usepackage{amssymb}
\usepackage{mathtools}
\usepackage{amsthm}

\usepackage[capitalize,noabbrev]{cleveref}

\usepackage{dsfont}
\usepackage{amsmath,amsfonts}
\usepackage{preamble} 

\usepackage{minitoc}

\theoremstyle{plain}

\begin{document}

\title{Approximate information maximization for bandit games}

%

\author{%
  Alex Barbier--Chebbah \\
  Institut Pasteur, Université Paris Cité\\
  Epimethee INRIA Paris France\\
  CNRS UMR 3571, Paris France \\
  \texttt{alex.barbier-chebbah@pasteur.fr} \\
  \And
  Christian L. Vestergaard \\
  Institut Pasteur, Université Paris Cité\\
  Epimethee INRIA Paris France\\
  CNRS UMR 3571, Paris France \\
  \AND
  Jean-Baptiste Masson \\
  Institut Pasteur, Université Paris Cité\\
  Epimethee INRIA Paris France\\
  CNRS UMR 3571, Paris France \\
  \And
  Etienne Boursier \\
  INRIA, Université Paris Saclay\\ 
  LMO, Orsay, France\\
}

\maketitle

\setcounter{tocdepth}{3}

\doparttoc 

\faketableofcontents 

\begin{abstract}
Entropy maximization and free energy minimization are general physics principles for modeling dynamic systems. Notable examples include modeling decision-making within the brain using the free-energy principle, optimizing the accuracy-complexity trade-off when accessing hidden variables with the information bottleneck principle \citep{tishbyInformationBottleneckMethod2000a}, and navigation in random environments using information maximization \citep{vergassolaInfotaxisStrategySearching2007}. Building on this principle, we propose a new class of bandit algorithms that maximize an approximation to the information of a key variable within the system. To this end, we develop an approximated, analytical physics-based representation of the entropy to forecast the information gain of each action and greedily choose the one with the largest information gain. This method yields strong performances in classical bandit settings. Motivated by its empirical success, we prove its asymptotic optimality for the multi-armed bandit problem with Gaussian rewards. 
Since it encompasses the system's properties in a single, global functional, this approach can be efficiently adapted to more complex bandit settings. This calls for further investigation of information maximization approaches for bandit problems. 
\end{abstract}

\section{Introduction}

Multi-armed bandit problems have attracted wide attention in the past decades. They embody the challenge of balancing exploration and exploitation and have been applied to various settings such as online recommendation \citep{bresler2014latent}, medical trials \citep{thompson1933likelihood}, dynamic pricing \citep{boer2015dynamic}, and reinforcement learning-based decision making \citep{silverMasteringGameGo2016,ryzhovKnowledgeGradientAlgorithm2012}. 
Besides the classic stochastic version of the multi-armed bandit problem, many subsequent extensions have been developed, providing richer models for specific applications. These extensions include linear bandits \citep{li2010contextual}, many-armed bandits \citep{bayati2020unreasonable}, and pure exploration problems such as thresholding bandits \citep{locatelliOptimalAlgorithmThresholding2016} or top-K bandits \citep{kalyanakrishnanPACSubsetSelection2012,kaufmannComplexityBestarmIdentification2016}.

In the classic setting, an agent chooses an arm at each time step and observes a stochastic reward. Since they only observe the payoff of the chosen arm, the agent should regularly explore suboptimal arms. This is often referred to as the exploration-exploitation trade-off. An agent can exploit its current knowledge to optimize gains by drawing the current empirically best arm or exploring other arms to potentially increase future gains.

Optimal strategies are characterized asymptotically by the Lai and Robbins bound \citep{lai1985asymptotically}. Among them, upper confidence bound  \citep[UCB,][]{auerUsingUpperConfidence2000, garivier_kl-ucb_2011} methods greedily pull the arm maximizing some tuned confidence index; Thompson sampling \citep{kaufmann2012thompson,agrawalThompsonSamplingContextual2013} relies on sampling mean rewards from a posterior distribution and chooses the arm with the largest random sample; deterministic minimum empirical divergence \citep[DMED,][]{hondaAsymptoticallyOptimalBandit2010a} builds on a balance between the maximum likelihood of an arm being the best and the posterior expectation of the regret. 

Even if these approaches efficiently utilize current available information, they do not aim directly to acquire more information. 
We highlight, however, the information directed sampling approach (IDS) of \citet{russoLearningOptimizeInformationDirected2014}, which relies on a measure of the information gain of the optimal actions. By leveraging an information measure that consistently captures the specific problem structure, IDS can address general classes of problems, particularly those with a complex information structure where classic bandit methods fall short. Surprisingly, IDS can even outperform UCB and Thompson sampling in classic bandit problems. However, like DMED, this method explicitly balances information gain with expected losses induced by exploration, and the efficiency of pure information-maximizing strategies thus remains to be proven.

Information-maximization approaches provide a decision-making strategy in which the agent tries to maximize information about one or more relevant stochastic variables. 
The information-maximizing principle has shown to be efficient in a broad range of domains \citep{heliasStatisticalFieldTheory2020a,parrActiveInferenceFree2022,hernandez-lobatoPredictiveEntropySearch2015,vergassolaInfotaxisStrategySearching2007} where decisions have to be taken in fluctuating or unknown environments. These domains include, e.g., robotics applications \citep{zhangMultiRobotSearchingSparse2015c}, where the ability to share approximate information improves collective decisions, and the search for olfactory sources in turbulent flows \citep{massonOlfactorySearchesLimited2013b,reddyOlfactorySensingNavigation2022}.

In the specific setting of bandit problems, information maximization has
 shown promising empirical results, and heuristic arguments support its asymptotic optimality \citep{reddyInfomaxStrategiesOptimal2016a, barbier-chebbahApproximateInformationEfficient2023}.  As IDS, they leverage information structure to provide a versatile decision framework with the capability to address various bandits settings. However, the efficiency of such a ``pure exploration'' strategy in terms of regret minimization has yet to be proven, and it has been previously argued that it would result in a linear regret \citep{russoLearningOptimizeInformationDirected2014}. 
Moreover, current information-based algorithms often rely on complex numerical integration, leading to high computational costs, a significant challenge that information-based methods must overcome.
In this context, we aim to leverage new strategies derived from information maximization principles focusing on global observables, \textit{i.e.} variables depending on more than one bandit arm, that alleviate the computational burden of numerical evaluation of complex functionals and to rigorously prove their efficiency.

\paragraph{Contributions.} 
Our main contribution is introducing a new class of asymptotically optimal algorithms that rely on approximations of a functional representing the current information of interest about the whole bandit system. This approach is based on the entropy of the posterior mean value of the best arm, for which we provide an approximate expression to enable robust, easily tunable, and extendable algorithms with a direct analytical formulation. We focus here on the multi-armed bandit problem with Gaussian rewards, for which we derive a simple approximate information maximization algorithm (\algonamese) and provide an upper bound on its pseudo-regret, ensuring that AIM is asymptotically optimal.
The information from each arm is incorporated in a unique entropy functional, which shows promise for tackling more complex bandit settings such as linear bandit or many-armed bandits. 
Thus, our main motivation is to design an analytical functional-based algorithmic principle, which can potentially address problems with more correlated information structures in the future. 
Additionally, another strength of \algoname lies in its short-time behavior, where it shows strong performances as we illustrate numerically for both Bernoulli and Gaussian rewards.

\paragraph{Organization.} In \cref{sec:setting}, we briefly review the K-armed bandit setting. \cref{sec:strategies} presents the general principle of information maximization approaches, originally inspired by both the information bottleneck principle and navigation in turbulent plumes. \cref{sec:results} upper bounds the regret of \algoname[,] showing it attains Lai and Robbin's asymptotic bound. In \cref{sec:expe}, the performance of \algoname is numerically compared with known baselines on multiple examples. Finally, \cref{sec:extensions} discusses extensions of \algoname to various bandit settings.

\section{Setting}\label{sec:setting}

We consider the classic $K$-armed stochastic bandit game. In each round $t$, the agent selects an arm $\Armt \in [K] = \{1, \ldots, K\}$ among a set of $K$ choices solely based on the rewards of the previously pulled arms. The chosen arm $k$ then returns a stochastic reward $\Rewardt(k)$, drawn independently of the previous rounds, according to a distribution $\nu_k$ of mean $\marmi{k}$. We denote by $\N_k(t)$  the number of times the arm~$k$ has been pulled. When clear from context, we omit the dependence on $t$ for simplicity.

The goal of the agent is to maximize its cumulative reward, or equivalently, to minimize its pseudo-regret up to round $T$, defined as
\begin{equation}\label{regretdef}
    \RegretT = \barm T - \sum_{t=1}^T  \Esp[ \marmi{\Armt}],
\end{equation}
where $\barm = \max_{i\in[K]}\marmi{i}$. Hence, the agent will optimize its choice of $\Armt$ relying on the previous observations up to~$t$. For a large family of reward distributions, the asymptotic pseudo-regret is lower-bounded for any uniformly good policy by
\begin{equation}\label{LaiRobbin}
\liminf_{T \to \infty} \frac{ \RegretT}{ \ln(T)} \geq  \sum_{k, \mu_k<\barm} \frac{\barm-\mu_{k}}{D_{\mathrm{KL}}(\nu_{k}\|\nu_{k^*})},
\end{equation}%
where $k^*\in\argmax_{i\in[K]} \mu_i,$, and $D_{\mathrm{KL}}(\nu_{k}\|\nu_{k^*}) $ denotes the Kullback-Leibler divergence between the reward distributions of the arms $k$ and $k^*$ \citep{lai1985asymptotically}. In the particular case of Gaussian rewards with equal variances, i.e., $\nu_i = \mathcal{N}(\mu_i,\varp)$, the Kullback-Leibler divergence is $D_{\mathrm{KL}}(\nu_{k}\|\nu_{k^*}) = (\mu^*-\mu_k)^2/(2\varp)$.

\section{Information maximization strategies}\label{sec:strategies}

Here, we introduce entropy-based, information maximization strategies and their underlying physical principles. We then detail approximations leading to an analytical and simplified entropy functional, which is the basis of the \algoname algorithm.

\subsection{Algorithm design principle: physical intuition}

We aim to design a functional encompassing the current available information of the full system. Inspired by the information maximization principle~\citep{vergassolaInfotaxisStrategySearching2007, reddyInfomaxStrategiesOptimal2016a} which has revealed effective in taxis strategies where the agent needs to find an emitting odour source \citep{martinezUsingInsectElectroantennogram2014, cardeNavigationWindbornePlumes2021, murlisOdorPlumesHow1992}, we rely on an entropic functional for policy decision. More precisely, we choose $\Smax$, the entropy of the posterior distribution of the value of the maximal mean reward, denoted~$\pmax$.

The algorithm relies on an arbitrary prior distribution on the arm mean rewards. 
With independent arm priors, the posterior distribution of the value of the maximal mean reward can be expressed as
\begin{equation}\label{pmaxgeneralexpression}
\begin{split}
\pmax(\val) \df\theta & = \df\Prob\left(\max_k \mu_k = \theta\mid \cF_{t-1}\right)= \sum_{\arm=1}^K \df\Prob(\marmi{\arm} = \val | \cF_{t-1}) \prod_{j\neq \arm} \Prob(\marmi{j} \leq \val | \cF_{t-1}),
\end{split}
\end{equation}
where $\cF_{t-1}=\sigma(X_1(a_1),\ldots,X_{t-1}(a_{t-1}))$ denotes the filtration associated to the observations up to time~${t-1}$.
The associated entropy reads 
\begin{equation}\label{Smax}
 \Smax = -\int_{\thetaDomain} \: \pmax(\val) \ln \pmax(\val) \df\val,
\end{equation}
where $\thetaDomain=[\marmi{\inf},\marmi{\sup}]$ is the support of $\pmax$ (which depends on the nature of the game and can be infinite). Note that, as exemplified by \cref{pmaxgeneralexpression}, $\pmax$ includes the arms' priors and directly depends on the reward distributions.\footnote{In the remainder of the paper, we consider an improper uniform prior over $\bR$, as often considered with Gaussian rewards.} 
The entropy, $\Smax$, is a measure of the information carried by all arms in a single functional, providing a global state description of the game.

Our policy aims to minimize the entropy of $\pmax$. For that, it greedily chooses the arm providing the largest expected decrease in entropy, conditioned on the current knowledge of the game,
\begin{equation}\label{greedygradient}
  \argmin_{k\in[K]} \E{\Smax(t)- \Smax(t-1) \mid \cF_{t-1},\Arm_{t} = \arm}.
\end{equation}
Similar to Thompson sampling, it relies on a Bayesian representation. Yet, it distinguishes itself by providing a deterministic decision procedure given past observations. We stress that $\Smax$ quantifies the available information about the average reward of the best arm. %
This choice contrasts with using the entropy of the probability of the best arm, which is known to overexplore and is suboptimal for regret minimization \citep{reddyInfomaxStrategiesOptimal2016a}. 
Because of this suboptimality, approaches based on the information on the best arm fix this concern by including the expected regret in the functional to favor exploitation \citep{russoLearningOptimizeInformationDirected2014}. Furthermore, we argue that by the definition of $\pmax$, the information carried by the arms' posteriors is sufficiently mixed to ensure an optimal behaviour, as proved in \cref{sec:results}. Since the policy aims to maximize the information about the best arm's mean, it mainly pulls the current best arm to learn more about its value. On the contrary, policies aiming to identify the best arm pull worse empirical arms more often because they are only concerned about the arms' order. 

The information maximization policy based on \cref{greedygradient} has been empirically shown to be competitive with state-of-the-art algorithms \citep{reddyInfomaxStrategiesOptimal2016a} and robust to variations of the prior \citep{reddyInfomaxStrategiesOptimal2016a,barbier-chebbahApproximateInformationEfficient2023} in classic bandit games. However, while \cref{greedygradient} can be numerically evaluated, it cannot be computed in closed form, preventing the gradient from being analytically tractable. This makes intricate to theoretically bound the regret even in the two-armed setting and it also prevents the policy's extension to more complicated bandit settings. 
Additionally, it induces a high computational cost \citep[a trait shared with IDS][]{russoLearningOptimizeInformationDirected2014}, which becomes disadvantageous when considering a large number of arms and at large times (when $\pmax$ is peaked), where one has to manage vanishing numerical precision, making the numerical integration even longer. Finally, the integral form of $\Smax$ prevents fine-tuning, which could prove crucial for achieving or surpassing the empirical state-of-the-art performances.

A second simplified and analytical functional mirroring $\Smax$ has to be derived to address these concerns. This analytical result strengthens the information maximization principle, both by providing novel algorithms that are analytical, tractable and computationally efficient while conserving the main advantages of the exact entropy \citep{reddyInfomaxStrategiesOptimal2016a} and by making theoretical analysis tractable. 

\subsection{Main elements of the entropy analytical approximation}

Here, we devise a set of approximations of $\pmax$ and $\Smax$ to get a tractable analytical algorithm. Given that the best empirical arm and the worse empirical arms have notably distinct contributions to $\pmax$ (Figure~\ref{fig:1pres}(a)), we approximate $\pmax$ while considering the current arms' order. We sort them based on their current posterior means, labelling the highest one as ${\maxa}$ (with an empirical reward of $\meana{\maxa}$) and  $\minasp = [K] \setminus \{\maxa\}$ the set of worse empirical arms. Of course, ${\maxa}$ might differ from the actual optimal arm $k^*$ due to the randomness in the observed rewards. 
We focus on approximating \cref{Smax,pmaxgeneralexpression} when the best empirical arm has already been extensively drawn more often than the other arms. 

The entropy is then decomposed into two tractable terms corresponding to distinct behaviors of $\pmax(\theta)$ when $\theta$ varies:
\begin{equation}\label{theoryS}
\begin{split}
\Sapp &= \Sc + \Stail ,
\end{split}
\end{equation}
The first term, $\Sc$, approximates the contribution around the mode of $\pmax$, while the second term, $\Stail$, quantifies the information carried by the tail of $\pmax$ (corresponding to high rewards, see \cref{fig:1pres}). Each of these terms then corresponds to a part of the entropy where the dominant term of \cref{pmaxgeneralexpression} is distinct (see \cref{firstanalyticalapprox} for details). 


More precisely, by denoting $ \parmi(\val)\df\val= \df\Prob(\marmi{i} = \val \mid \cF_{t})$ the mean posterior density of the associated arm $i$, the tail term is approximated as: 
\begin{equation}\label{Stailexp}
\Stail = - \sum_{\mina\in\minasp} \int_{\teqmin}^{\bsup} \parmin(\val) \ln \parmin(\val) \df\val.
\end{equation}
where $\teqmin$, given in \cref{SifinalformdeltaSi}, approximates $\tpmin$, the value of $\theta$ where the empirical best arm $\maxa$ and the selected worse arm $\mina$ have the same probability of being the best arm (see red and orange curves in \cref{fig:1pres}(b)). Here, $\parmin(\val)$ is the posterior density of the current worse arm evaluated at~$\val$. Roughly, because the better empirical arm has been predominantly drawn, $\parmax(\val)$ decays faster than $\parmin(\val)$, resulting in a tail term (see \cref{Stailexp}) whose main contribution is the worse empirical arm. The approximate entropy of the body component is: 
\begin{equation}\label{Sbodyexp}
\Sc = - \int_{\thetaDomain}  \big( 1 - \sum_{\mina\in\minasp} [1- \cumin(\val)] \big) \parmax(\val)  \ln\parmax(\val) \df\val 
\end{equation}
where $\cumi(\val)= \Prob( \val  > \marmi{i} \mid \cF_t )$ is the cumulative posterior probability of the mean of the arm $i$. \cref{Sbodyexp} is the leading-order term of the mode of $\pmax$, which is mainly contributed to by the best empirical arm.

This approximation of \cref{pmaxgeneralexpression} is good when the best empirical arm has been extensively drawn compared to the worse empirical ones, corresponding to the situation encountered asymptotically for uniformly good algorithms. Surprisingly, the approximation captured by \cref{theoryS} is still accurate enough outside this asymptotic regime to provide a high-performance decision scheme.

\begin{figure}[htpp]
\centering
\includegraphics[scale=0.65]{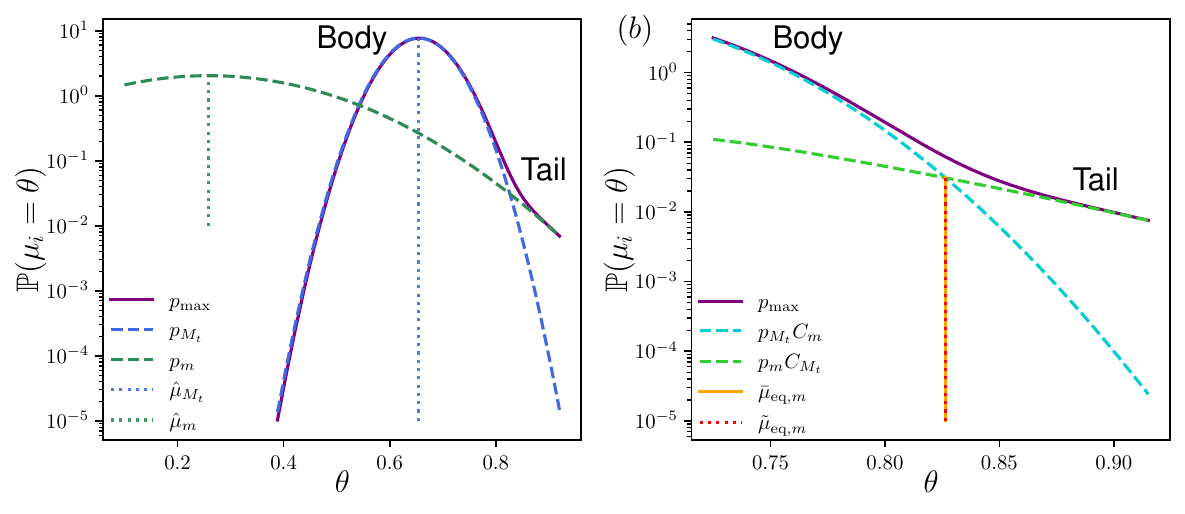}
\caption{\textbf{(a)} Posterior distributions of a two-armed bandit with Gaussian rewards. The dotted lines represent the individual posterior distributions of each arm, $\parmax$ and $\parmin$, while the continuous line represents the posterior of the maximum mean reward of all arms, $\pmax$ (\cref{pmaxgeneralexpression}). 
\textbf{(b)} Zoom of \textbf{(a)} around the point $\tpmin$ where both arms have the same posterior probability of being the best one. $\parmax\cumin$ ($\parmin\cumax$) is the probability that the maximal value is given by the better (worse) empirical arm, and $\teqmin$ is the approximation to $\tpmin$ given in \cref{{SifinalformdeltaSi}}.
}
\label{fig:1pres}
\end{figure}

For Gaussian reward distributions, one can derive an analytical expression for $\teq$ (see \cref{refteqapprox} for details), \cref{Stailexp,Sbodyexp} can be computed exactly  (see \cref{reftailgauss}). However, at this stage, even if we have already obtained a closed-form expression for $\Smax$, it remains too involved to directly compute its exact (discrete) gradient for our decision policy. To finally derive a simplified gradient, we retain only the asymptotic terms of \cref{Stailexp,Sbodyexp} and of the obtained gradient (see \cref{Sigradientgaussian} for derivation details). Finally, the expression of our approximate difference of gradients of the entropy, whether calculated along a given worse empirical arm $k$ or along the best empirical arm, reads:
\begin{equation}\label{Sgradientsimp}
\begin{split}
&\Delta_{\armmax, k} =  \frac{1}{2}\ln ( \frac{\Nmax}{\Nmax + 1}) +   \frac{1}{2 \Nmax}\min\left( \frac{1}{2}\sum_{\mina}^{\minasp} \erfc \left(  \delta \teqmin \right), \, 1-\frac{1}{K}\right)  \\& +  Q(\Nk^{-1},  \ln(\Nmax), \delta \teqk)e^{-\delta \teqk^2} +  \sum_{\mina}^{\minasp} P( \Nmin^{1/2}, \Nmax^{-1}, \ln(\Nmax), \delta \teqmin ) e^{-\delta \teqmin^2} \\
\end{split}
\end{equation}
where $Q$ and $P$ are polynomials given in \cref{SifinalformdeltaSi} and $\delta \teqi = \frac{\sqrt{\Ni} (\teqi - \meana{i})}{\sqrt{2\varp}}$ are standardized variables with $\teqi$ given in \cref{SifinalformdeltaSi}.
In words, $\Delta_{\armmax, k}$ approximates the difference 
\begin{equation*}\begin{split}
    \Delta_{\armmax, k} \approx & \phantom{-}\E{\Smax(t+1)\mid \cF_{t},\Arm_{t+1} = \armmax} - \E{\Smax(t+1)\mid \cF_{t},\Arm_{t+1} = k},
\end{split}
\end{equation*}
which is directly related to greedily maximizing the entropy decrease, described in \cref{greedygradient}. The decision procedure can be summarized as follows: if $\Delta_{\armmax, k}$ is negative for all $k \in \minasp$, the better empirical arm is chosen as it reduces the most the expected value of the approximate entropy. Inversely, if at least one value $\Delta_{\armmax, k}$ is positive, the arm $k$ maximizing $ \Delta_{\armmax, k}$ is chosen.

In conclusion, we have derived an analytical expression for the information available about the maximum expected reward of all arms. We isolated an analytically tractable gradient acting as a decision procedure that eluded previous approximated information derivations~\citep{barbier-chebbahApproximateInformationEfficient2023}. Our
 scheme leads to an efficient numerical implementation by eliminating numerical integrals, substantially improving the computational speed of information maximization, a crucial challenge for information methods, which is also stressed by \citet{russoLearningOptimizeInformationDirected2014} for the IDS algorithm.
We now provide the full implementation of \algoname and bound its regret in the next section.

\subsection{Approximate information maximization algorithm }

The pseudo-code for the \algoname algorithm is presented in \cref{alg:bandit_algorithm} below.

\begin{algorithm}[htbp]
\DontPrintSemicolon
\caption{\algoname Algorithm for $K$ Gaussian arms}
\label{alg:bandit_algorithm}
Draw each arm once, observe reward $\Rewardt(t)$ and update statistics
$\meana{t}$\;
\For(\tcp*[f]{Arm selection}){$t=K+1$ \KwTo $T$}{
    \lIf{$\Nmax \leq \Nmin$}{
    $\Armt\gets\armmax$}
    \Else{
    $\armmax \gets \argmax_{k\in[K]} \meana{k}$\;
    Evaluate $ \mina = \argmax_{k \in \minasp} \Delta_{\armmax, k}$
    following \cref{Sgradientsimp} \;
    \lIf{$\Delta_{\armmax, \mina} \leq 0$}{
    $\Armt\gets\armmax$}
    \lElse{
    $\Armt\gets\armmin$}
    Pull $\Armt$ and observe $\Rewardt(\Armt)$\;
    }
    $\meana{\Armt} \leftarrow \frac{\meana{\Armt}\ndraw{\Armt} +\Rewardt(\Armt)}{ \ndraw{\Armt} + 1} $,
    $\ndraw{\Armt} \leftarrow \ndraw{\Armt} + 1 $\tcp*[f]{Update statistics}
}
\end{algorithm}

The best empirical arm is drawn by default if there exists one empirical arm $\mina$  that has been drawn more frequently ~${\Nmax \leq \Nmin}$. In such a case, both entropy components in \cref{theoryS} are mainly contributed to by $\armmax$.  

\section{Regret bound}\label{sec:results}
This section provides theoretical guarantees on the performance of \algoname[.] More precisely, \cref{thm:main} below states that \algoname is asymptotically optimal on the multi-armed bandits problem with Gaussian rewards.
\begin{theorem}\label{thm:main}
    For Gaussian reward distributions with variance $\sigma^2$, the regret of \algoname satisfies for any mean vector $\pmb{\mu}\in \bR^K$
    \begin{equation*}
        \limsup_{T\to\infty} \frac{\RegretT}{\ln (T)} \leq \sum_{k,\mu_k<\mu^*}\frac{2\sigma^2}{\mu^*-\mu_k},
    \end{equation*}
    where $\mu^* = \max_{k\in[K]}\mu_k$.
\end{theorem}
With Gaussian rewards, the asymptotic regret of \algoname thus exactly reaches the lower bound of \citet{lai1985asymptotically} given by \cref{LaiRobbin}. A non-asymptotic version of \cref{thm:main} is given by \cref{thm:app} in \cref{app:proof}. We briefly sketch the proof idea below and refer to  \cref{app:proof} for the complete proof.

\paragraph{Sketch of the proof.}
We assume for sake of clarity in this sketch that $\mu_1>\mu_k$ for any $k\geq 2$. 
The structure of the proof is similar to the one found in \citet{kaufmann2012thompson}. In particular, the first main step shows that the optimal arm is pulled at least $\sqrt{t}$ times with high probability. This result holds because otherwise, the contribution of arm $1$ to the tail of the distribution would dominate the contribution of other arms in the approximate information. In that case, pulling the first arm would naturally lead to a larger decrease in entropy, which ensures that the optimal arm is always pulled a significant amount of times.

Then, we only need to work in the asymptotic regime where arm $1$ is pulled at least $\sqrt{t}$ times and we aim at bounding the number of pulls on the arm $k\geq 2$. Additionally, we restrict ourselves to a large number (in $\log(T)$) of pulls on arm $k$ and automatically count the pulls before that point in the regret. As a consequence, we can show that with high probability: 
\begin{equation*}
    \meana{\armmax} \geq \mu^* - \sqrt{\frac{6\varp\ln t}{\sqrt{t}}} \quad \text{and} \quad \meana{k}\leq \mean_k + \varepsilon
\end{equation*}
for some arbitrary $\varepsilon>0$. An important property of entropy is that it approximates the behaviour of the bound of \citet{lai1985asymptotically}. 
More precisely, in the asymptotic regime, the difference of the entropy increments behaves as
\begin{equation}\label{eq:asymptoticentropy}\begin{aligned}
    \Delta_{\armmax, k} \approx -\frac{1}{2N_{\armmax}} + Q(N_{k})e^{-\frac{N_{k}(\mu_1-\mu_k)^2}{2\varp}} + \sum_{i\neq\armmax}P(N_i)e^{-\frac{N_{i}(\mu_1-\mu_i)^2}{2\varp}},
    \end{aligned}
\end{equation}
where $Q_k$ and $P_i$ are polynomials that also depend on extra variables (see Equation~\ref{Sgradientsimp}). Manipulating these polynomial terms altogether is intricate, but we can still show that if the arm $k$ is pulled, this means the term $e^{-\frac{N_{k}(\mu_1-\mu_k)^2}{2\varp}}$ somewhat dominates the other exponents in the sum of \cref{eq:asymptoticentropy}. This then implies that $N_{k}$ is of order at most $\frac{2\varp\ln T}{(\mu^*-\mu_k)^2}$, as arm $k$ is only pulled if $ \Delta_{\armmax, k}\geq 0$. 
\hfill \qed

Our policy is deterministic at each time step while displaying a logarithmic regret, showing that intuitions from~\citet{russoLearningOptimizeInformationDirected2014} of linear regrets for stationary (in the sense they only depend on the posterior distribution) deterministic algorithms was inexact. Moreover, our regret bound is frequentist, in opposition to the Bayesian regret bound obtained for IDS \citep{russoLearningOptimizeInformationDirected2014}. As a consequence, \algoname does not need a well-specified prior: using a uniform prior, as done in our work, is a well-suited choice. 
Also, the required form of the entropy for the proof is general. The algorithm yields an optimal regret as long as we are guaranteed that the optimal arm is pulled a significant amount of times with high probability and that the asymptotic regime behaves as \cref{eq:asymptoticentropy}. Hence,  \cref{thm:main} will hold for a large family of entropy approximations \citep[and likely for generalizations to free energies too, as in][]{massonOlfactorySearchesLimited2013b} as long as the approximation is accurate enough to not yield trivial behaviors in the short time regime. Additionally, the approximate framework devised here allows fine-tuning the formulas to improve short-time performance all the while ensuring asymptotic optimality by keeping the correct asymptotic terms.

\section{Experiments}\label{sec:expe}

This section investigates the empirical performance of \algoname (\cref{alg:bandit_algorithm}) on numerical examples. All details of the numerical experiments are given in \cref{app:expe}.

We start by considering two arms with Gaussian rewards~\citep{hondaAsymptoticallyOptimalBandit2010a} of unit variance and means $\marmi{\arm}$ drawn uniformly from $[0,1]$. \cref{fig:1} compares the Bayesian regret (i.e., the regret averaged over all values of $(\marmi{1},\marmi{2})$ in $[0,1]\times[0,1]$) of \cref{alg:bandit_algorithm} with the state-of-the-art algorithms  UCB-tuned, Thompson sampling, Thompson sampling+, KLUCB++, and MED \citep{kaufmannThompsonSamplingAsymptotically2012a,pilarskiOptimalPolicyBernoulli2021,cappeKullbackLeiblerUpper2013, jinFiniteTimeRegretThompson2022, hondaAsymptoticallyOptimalPolicy2011a, menardMinimaxAsymptoticallyOptimal2017}. We refer to \cref{app:otheralgos} for an overview and detailed descriptions of these bandit algorithms. The Bayesian regret of \algoname empirically scales as $\log\left(T\right)$. Its long-time performance matches Thompson sampling, as implied by \cref{thm:main}, while relying on a (conditionally) deterministic decision process. Additionally, \algoname outperforms Thompson sampling at both short and intermediate times (see \cref{app:moredata} for finer measurements). \algoname particularly outperforms Thompson sampling when the arms are difficult to distinguish due to their mean rewards being close (see examples in \cref{app:fixedmeanexpe} with single instance regret experiments).

\begin{figure}[htbp]
\centering
\includegraphics[scale=0.6]{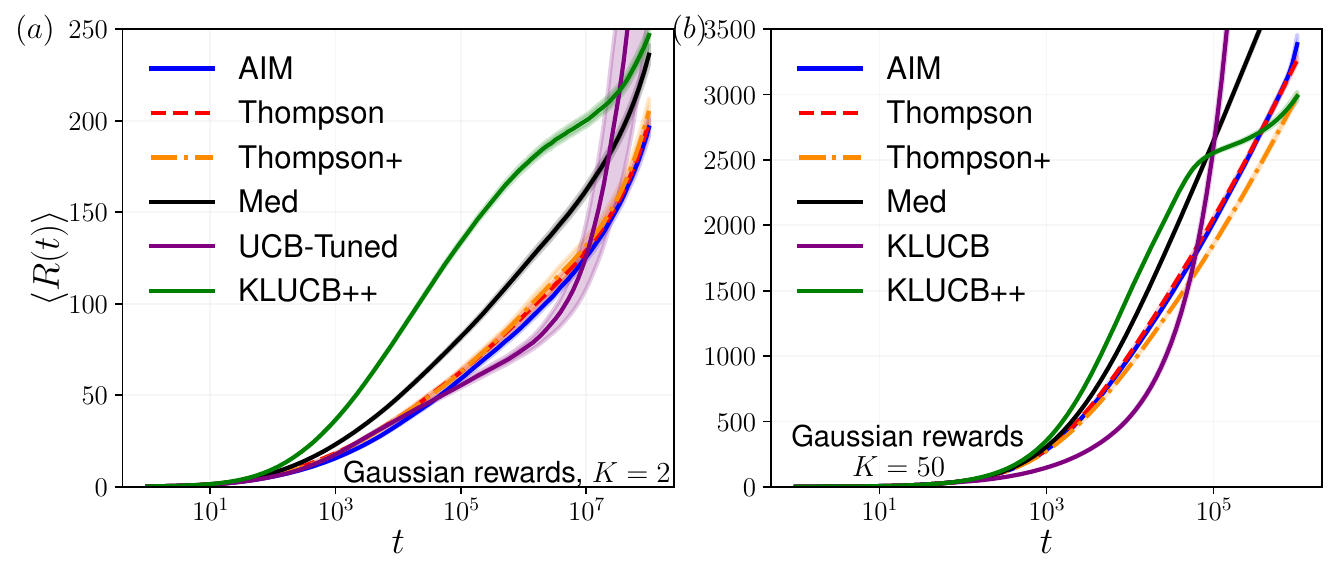}
\caption{ Evolution of the Bayesian regret for (a) 2-armed and (b) 50-armed bandit with Gaussian rewards under a uniform mean prior. Regret is averaged over $8000$ for (a) and  $2000$ runs for (b) Confidence intervals show the standard deviation.}
\label{fig:1}
\end{figure}

\algoname yields strong performance in both two-armed Gaussian and 50-armed Gaussian rewards case, as predicted by our theoretical analysis. We now aim to extend our method to other bandit settings. \cref{fig:3} presents the performance of \algoname when adapted to Bernoulli rewards~\citep{pilarskiOptimalPolicyBernoulli2021} with arm means drawn uniformly in $[0,1]$. This adaptation is described in detail in \cref{sec:extensions} below. The performance of \algoname is comparable to Thompson sampling here. Additionally, \algoname performs comparably to Thompson sampling for close mean rewards (see \cref{app:fixedmeanexpe2}). Additionally, for $50$ arms with Bernoulli rewards AIM's short-time efficiency is comparable to Thompson sampling, and it is significantly more efficient at intermediary times while showing the same logarithmic scaling at long times as Thompson sampling. 

Hence, our algorithm shows strong empirical performances compared to state-of-the-art baselines for both Bernoulli and Gaussian rewards while providing outstanding effectiveness when facing multiple arms with Bernoulli rewards. Experiments suggest that AIM displays the same typical worst-case regret as Thompson sampling (which is minimax optimal up to $\sqrt{\ln K}$ for sub-Gaussian rewards), but proving a theoretical bound remains challenging and left for future work.  Of note, similar observations are drawn in \cref{app:fixedmeanexpe,app:fixedmeanexpe2} for non-Bayesian versions of the regret, with fixed bandit instances. These observations support the robustness of \algoname and its potential for extensions to more complex bandit settings.

\begin{figure}[htbp]
\centering
\includegraphics[scale=0.6]{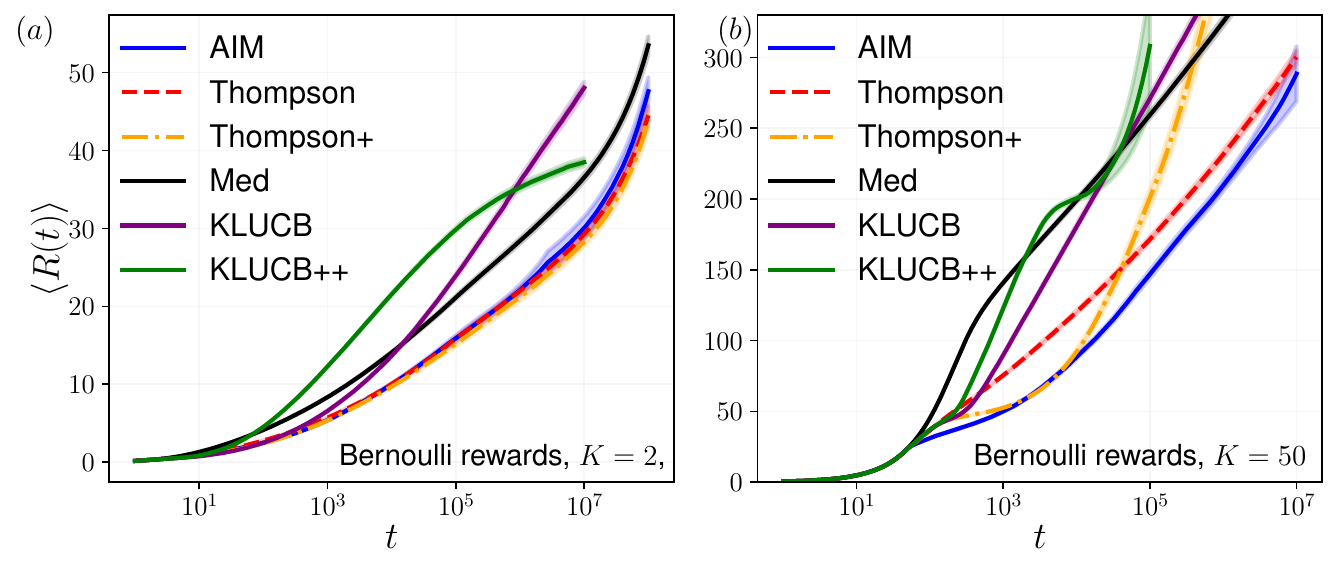}
\caption{Evolution of the Bayesian regret for (a) 2-armed and (b) 50-armed bandit with Bernoulli rewards under a uniform mean prior. The regret is averaged over $16000$ runs for (a) and  $2000$ runs for (b). Confidence intervals show the standard deviation.}
\label{fig:3}
\end{figure}

\section{Extensions}\label{sec:extensions}

We apply our information maximisation approach to Bernoulli bandits both with two and with many arms, where it shows strong empirical performances (see \cref{fig:3} above). 
This section describes the extensions of \algoname to this case and discusses potential extensions to more general bandit settings. 

\paragraph{Exponential family bandits.}

Since \cref{pmaxgeneralexpression} explicitly relies on the arms' posterior distributions, information maximization methods can be directly extended to various reward distributions. 
In particular, when the reward distributions belong to the exponential family \citep[see][and \cref{app:exponential} for details on such distributions]{kordaThompsonSampling1dimensional2013}, an asymptotic and analytical expression of the entropy can be derived for the case of uniform priors (see \cref{expofamily} for more details), yielding
%
\begin{equation}\label{Sappgeneral}
\begin{split}
\Sapp &= \frac{1}{2}\ln ( 2\pi  \Vi)  \left[ 1 - \frac{ e^{- \Nmin \kullfct(\thetamin,\thetaeq)}}{ \Nmin \partial_2 \kullfct(\thetamin,\thetaeq) \sqrt{2 \pi \Vmin}} \right] +   \frac{\kullfct(\thetamin,\thetaeq)  e^{- \Nmin \kullfct(\thetamin,\thetaeq)}}{ \partial_2\kullfct(\thetamin,\thetaeq) \sqrt{2 \pi \Vmin}  }.
\end{split}
\end{equation}
Here $\kullfct(\hat{\theta}_{i},\thetaeq)$ is the \kulleib divergence between the reward distribution parameterized by $\hat{\theta}_{i}$ and $\thetaeq$ where $\theta$ is the family parameter, and $ \partial_2\kullfct$ denotes its derivative w.r.t.\ the second variable. All the steps leading to \cref{Stailexp,Sbodyexp} in \cref{sec:strategies} are not specific to Gaussian rewards. The main difference lies in their asymptotic simplifications obtained afterwards with Laplace's method. Our implementation of \algoname to Bernoulli rewards (a specific case of the exponential family) with \cref{Sappgeneral} shows comparable performance to state-of-the-art algorithms (see \cref{fig:1}), supporting its adaptability to general settings.  We believe that \algoname should be optimal for all exponential family reward distributions and general prior distributions and that similar proof techniques can be used (see \cref{expofamily} for a detailed discussion). However, significant work still remains to ensure that the asymptotic regime, where all arms have been sufficiently drawn, is reached for any reward distribution and will be addressed in future work.

\paragraph{Other bandit settings.} Here, we provide a quick overview of several other bandit settings for which approximate information maximization, adapted to the specific bandit problem, should provide efficient algorithms.
First, we emphasize that \algoname['s]{} partitioning between body and tail components remains relevant even when dealing with heavy-tailed \citep{lee2023optimality} or non-parametric reward distributions \citep{baudrySubsamplingEfficientNonparametric2020}. It should thus be able to provide strong guarantees in these settings, similarly to Thompson sampling. Secondly, let us stress that information can also be quantified for unpulled arms, which may prove crucial when facing large numbers of arms. The agent could quantify the information of the ``reservoir'' of unpulled arms to anticipate the information gained from exploring these unpulled arms. Additionally, if the agent has access to the remaining time, it can not only evaluate the expected information gain when pulling an arm for a single round but also evaluate the information gain of multiple pulls of the same arm. We believe that such a consideration might be pivotal when facing many arms, since the limited amount of time does not allow to pull all the arms \citep{bayati2020unreasonable} sufficiently. Thirdly, in linear bandits, where arms are correlated with each other \citep{li2010contextual}, \algoname will be efficient because pulling a specific direction provides information on correlated directions, the shared information gain could be leveraged by information-based methods to yield strong performances. Finally, we could consider pure exploration problems \citep{bubeckPureExplorationFinitelyarmed2011a, locatelliOptimalAlgorithmThresholding2016,kalyanakrishnanPACSubsetSelection2012} where the agent's goal is directly linked to an information gain, thus making the information maximization principle an inherent candidate when a suitable entropy is derived from the underlying bandit structure and problem objective.

A last advantage of \algoname lies in its possible extension to multiple constraints that would be introduced using Lagrange multipliers (or borrowed from physics reasoning by defining free energy), further improving its adaptability to various settings and specific requirements.

\section{Conclusion}\label{sec:conclusion}

This paper introduces a new algorithm class, Approximate Information Maximization (AIM), which leverages approximate information maximization of the whole bandit system to achieve optimal regret performances. This approach builds on the entropy of the posterior of the arms' maximal mean, from which we extract a simplified and analytical functional at the core of the decision scheme. It enables easily tunable and tractable algorithms, which we prove to be optimal for multi-armed Gaussian bandits. Numerical experiments for Bernoulli rewards with two or several arms emphasize the robustness and efficiency of AIM. An additional strength of \algoname lies in its efficiency at short times and when the arms have close mean rewards where it outperforms existing state-of-the art. Further research should focus on adjusting the information maximization framework to more complex bandit settings, including many-armed bandits, linear bandits and thresholding bandits, where appropriately selected information measures can efficiently apprehend the games' structure and correlations.

\newpage


\bibliographystyle{plainnat}
\bibliography{biblio.bib,Aistats_article}

\clearpage



\appendix
\onecolumn

\addcontentsline{toc}{section}{Appendix} 
\part{Appendix} 
\parttoc 





\section{Towards an analytical approximation of the entropy}\label{analyticalapprox}

In this section, we recapitulate all the steps leading to the analytical expression constitutive of our \algoname algorithm. We stress, that it involves exact derivations but also simplifications to considerably simplify the final form of AIM. Therefore, alternative approaches could  lead to a slightly different version of AIM. However, our chosen method retains the essential features which emerges in the asymptotic regime while providing a simple version of AIM.  

\subsection{The partitioning approximation}\label{firstanalyticalapprox}

We start by commenting on the partition scheme and the approximations leading to the following body/tail expressions. We first recall the expression for $\pmax(\val)$, with the arms ordered along $\armmax$,
\begin{equation}\label{Sipmaxgeneralexpression}
\begin{split}
\pmax(\val) &= \left(  \parmax(\theta)  \prod_{ \mina }^{\minasp} \cumin(\theta) + \sum_{\mina}^{\minasp} \cumax(\theta) \parmin(\theta)  \prod_{j\neq \mina}^{\minasp} \cumj(\theta)  \right).
\end{split}
\end{equation}
where we remind that $\cumin(\theta)$ is the cumulative posterior probability of the mean of arm $m$

Because of its dependency along all arms, there is no unique dominant term in  \cref{Sipmaxgeneralexpression}, and distinct regimes emerge depending on $\theta$ and the state of the game. Then, we assume to isolate distinct regimes contributing asymptotically to the entropy while significantly simplifying them. It will considerably simplify the derivation of an analytical expression for the body/tail components in the next section. The next paragraphs will then present heuristic arguments justifying our simplification scheme.

We start by rewriting the exact entropy expression isolating $\armmax$:
\begin{equation}\label{SiSmaxmultiexpression}
\begin{split}
&\Smax = - \int_{\thetaDomain}  \parmax(\theta)   \prod_{ \mina }^{\minasp} \cumj(\theta)  \ln \left(  \parmax(\theta)  \prod_{ j}^{\minasp}  \cumj(\theta) + \sum_{ \mina }^{\minasp} \parmin(\theta) \cumax(\theta) \prod_{j\neq \mina}^{\minasp}  \cumj(\theta)  \right) \df\theta
\\ & - \sum_{\mina}^{\minasp}\int_{\thetaDomain}  \parmi(\theta) \cumax(\theta)  \prod_{j\neq \mina}^{\minasp} \cumj(\theta)   \ln \left(  \parmax(\theta)  \prod_{j}^{\minasp} \cumj(\theta) + \sum_{\mina }^{\minasp} \parmin(\theta) \cumax(\theta) \prod_{j\neq \mina}^{\minasp} \cumj(\theta)  \right) \df\theta.\\
\end{split}
\end{equation}
Let us briefly comment on the different contributions to \cref{SiSmaxmultiexpression}. We aim to keep the leading orders of $\pmax(\val)$ when $\Nmax \gg \Nmin \gg 1$ and $\meanmax > \meanmin$ for all $\mina$ in the set of current worse empirical arms $\minasp$. Here, the posterior distributions are assumed uni-modal. The first term is the leading order in the vicinity of the mode of $\meanmax$. Also, since $\Nmax > \Nmin$,  $\parmax(\val)$ is more concentrated than all $\parmin(\val)$, resulting in the dominance of the second term in the distribution's tail (\textit{i.e.}, for high rewards). 

We now decompose the entropy in the body/tail components defined in the main text, the first term of \cref{SiSmaxmultiexpression} will form the body component, and the sum over all worse empirical arms will compose the tail. We now define $\tpmin$ the intersection point associated to the arm $\mina$ verifying $\cumin(\tpmin) \parmax(\tpmin) = \cumax(\tpmin)  \parmin(\tpmin)$. Then, in the asymptotic regime, $\tpmin$ will verify  $\parmin(\val) \gg \parmax(\val)$ for $\theta > \tpmin$ and $\parmin(\val) \ll \parmax(\val)$ for $\theta < \tpmin$. Again, we will assume to neglect the transition regime where $\tpmin \sim \theta$ where both distributions are of the same order because it is narrow (in the asymptotic regime) and has very little influence on the total value of the entropy.

To get the body component, we consider the first term of \cref{SiSmaxmultiexpression}. We neglect all the inner terms inside the logarithm which is then dominated by $\armmax$. Next, by noticing that $ \cumi(\theta) \approx 1 $ is in the vicinity of $\meanmax$, we make a first-order expansion of the remaining product along all the worse empirical arms. Since the inner term of the body component is negligible for $\theta > \min( \{ \tpmin, \mina \in \minasp \})$ (because of its dependency along $\parmax$), we ignore that our simplification is no more valid in this specific regime without loss of consistency. Taken together we obtain the body expression of the main text:

\begin{equation}\label{SiSP1}
\begin{split}
\Sc  = - \int_{\thetaDomain}  \big( 1 - \sum_{\mina}^{\minasp} [1- \cumin(\val)] \big) \parmax(\val)  \ln\parmax(\val) \df\val.\\
\end{split}
\end{equation}

Then, we consider the additional terms (each denoted as $\mina$) in \cref{SiSmaxmultiexpression}. First, each term of the sum is negligible to the first one for $\theta < \tpmin$, we then only keep the upper part of the integral where $\theta > \tpmin$. Because $\Nmin \gg 1$ and $\theta > \min( \{ \tpmin, i \neq \armmax \}) > \meanmax > \meana{j}$, we approximate all the cumulative by one. Finally, to get a simplified expression for the increment, we assume to neglect all the posterior distributions except for $\parmi(\val)$ inside the logarithm of the $i$-th term and approximates $\tpmin$ (see next section) which leads to the tail expression:

\begin{equation}\label{SiSP1_bis}
\begin{split}
\Stail = - \sum_{\mina }^{\minasp} \int_{\teqmin}^{\bsup} &\parmin(\val) \ln \parmin(\val) \df\val.
\end{split}
\end{equation}

note that some of these posterior distributions ($ j \neq \mina, \armmax$) are not negligible compared to $ \parmin (\val) $ at a given $ \val$. However, this cross-information between current suboptimal arms is asymptotically negligible regarding the decision procedure (which largely resumes as balancing exploiting the best empirical solution compared to exploring worse empirical arms) while unnecessarily complicating the increment evaluation.

Finally, we obtain the full expression of the entropy approximation:

\begin{equation}\label{KmGaussian1}
\Sapp= - \int_{\thetaDomain}  \bigg( 1 - \sum_{\mina }^{\minasp} [1- \cumin(\val)] \bigg) \parmax(\val)  \ln\parmax(\val) \df\val - \sum_{\mina }^{\minasp} \int_{\teqmin}^{\bsup} \parmi(\val) \ln \parmi(\val) \df\val.
\end{equation}

\subsection{ Asymptotics of the intersection point}\label{refteqapprox}
    
In this section, we derive the asymptotic expression of the intersection point (defined above as $\teqmin$) where the distributions $\cumin(\teqmin) \parmax(\teqmin)$ and $\cumax(\teqmin)  \parmin(\teqmin)$ intersect (at their highest value if they intersect more than once). Here, we consider Gaussian rewards and the intersection between $\armmax$ and a given worse empirical arm denoted $\mina$. The exact equation verified by the intersection point $\tpmin$ is:
\begin{equation}\label{Sithetaeq1}
\begin{split}
& \frac{ \sqrt{\Nmax} e^{- \frac{\Nmax(\tpmin- \meanmax)^2}{2 \sigma^2} }}{\sqrt{2\pi \sigma^2}}   \frac{1}{2}\left[ 1 + \erf \left(\frac{\sqrt{\Nmin}(\tpmin- \meanmin)}{\sqrt{2 \sigma^2}} \right) \right]=  \\& \hspace{3cm}  \frac{ \sqrt{\Nmin} e^{- \frac{\Nmin(\tpmin - \meanmin)^2}{2 \sigma^2} }}{\sqrt{2\pi \sigma^2}}   \frac{1}{2}\left[ 1 + \erf \left(\frac{\sqrt{\Nmax}(\tpmin - \meanmax)}{\sqrt{2 \sigma^2}} \right) \right] .\\
\end{split}
\end{equation}

Taking the logarithm of \cref{Sithetaeq1} and normalizing the last term leads to:
\begin{equation}\label{Sithetaeq2}
\begin{split}
&   \frac{\Nmin(\tpmin - \meanmin)^2}{2 \sigma^2}  -  \frac{\Nmax(\tpmin - \meanmax)^2}{2 \sigma^2}  +  \frac{1}{2}\ln \frac{\Nmax}{\Nmin} + \ln \left[ \frac{1 + \erf \left(\frac{\sqrt{\Nmin}(\tpmin - \meanmin)}{\sqrt{2 \sigma^2}} \right)}{ 1 + \erf \left(\frac{\sqrt{\Nmax}(\tpmin - \meanmax)}{\sqrt{2 \sigma^2}} \right)} \right] = 0.
\end{split}
\end{equation}

The distributions are uni-modal, and assuming that $\meanmax>\meanmin$, $\Nmax > \Nmin$ and recalling that $\tpmin$ is the highest intersection, we get that $\tpmin >\meanmax > \meanmin$. Both error functions are then bounded in $[0,1]$, making the last term bounded as well.
We then approximate $\tpmin$ with $\teqmin$ by neglecting the last term, which leads to the following solution:
\begin{equation}\label{Sithetaeq6}
\begin{split}
&\teqmin = \meanmax + \frac{\Nmin(\meanmax -\meanmin)}{\Nmax-\Nmin}+  \sqrt{ \frac{\Nmax \Nmin}{(\Nmax - \Nmin)^2} ( \meanmax-\meanmin)^2  + \frac{\sigma^2 }{\Nmax- \Nmin}\ln \left( \frac{\Nmax}{\Nmin} \right)  } .
\end{split}
\end{equation}
Note that \cref{Sithetaeq6} relies on both $\meanmax > \meanmin$ and $\Nmax>\Nmin$.  For $\Nmax \leq \Nmin$, even if the above $\teqmin$ can be computed, it does not quantify the tail contribution. As a matter of fact, for $\Nmax \leq \Nmin$, the tail is always dominated by $\parmax$, which means that it has already been included in the main mode $\Sc$. Then, in this specific configuration, we take the contribution of the arm $\mina$ to $\Stail$ equals to $0$ (in other words $\teqmin = \bsup$).

\subsection{Closed-form expressions for the main mode's contribution}\label{SiScgaussian}

Here, we derive the $\Sc$ expression given in the main text for Gaussian rewards distribution. 
Inserting the Gaussian form of the posterior into \cref{Sbodyexp} gives:

\begin{equation}\label{SiSb1}
\begin{split}
\Sc = &-\int_{-\infty}^{+\infty} \frac{ \sqrt{\Nmax} e^{- \frac{\Nmax(\val - \meanmax)^2}{2 \sigma^2} }}{\sqrt{2\pi \sigma^2}}  \left(-\frac{1}{2}\ln( \frac{2 \pi \sigma^2}{\Nmax}) - \frac{\Nmax(\val - \meanmax)^2}{2 \sigma^2} \right) \times \\ &\hspace{3cm}  \left( 1 - \sum_{\mina}^{\minasp} \frac{1}{2}\left[ 1 - \erf \left(\frac{\sqrt{\Nmin}(\val - \meanmin)}{\sqrt{2 \sigma^2}} \right) \right] \right) \df\val,
\end{split}
\end{equation}

We integrate the constant part of the first term, denoted $T_1$ by the use of the following identity~\citep{ngTableIntegralsError1969}:
\begin{equation}\label{Sierfint}
\begin{split}
\int_{-\infty}^\infty & \bigg[1+\erf \left(\frac{\theta-\theta_1}{\sqrt{2V_1}} \right) \bigg] \frac{ e^{-\frac{(\theta-\theta_2)^2}{2 V_2}}}{\sqrt{2 \pi V_2}}\df\theta = 
\left[1  +  \erf \left(\frac{\theta_2-\theta_1}{\sqrt{2} \sqrt{V_2 + V_1}} \right) \right],
\end{split}
\end{equation}
which leads to 

\begin{equation}\label{SiSb2}
\begin{split}
T_1 &= \frac{1}{2}\ln\left(\ \frac{2 \pi \sigma^2}{\Nmax} \right) \int_{-\infty}^{+\infty} \frac{ \sqrt{\Nmax} e^{- \frac{\Nmax(\val - \meanmax)^2}{2 \sigma^2} }}{\sqrt{2\pi \sigma^2}}  \bigg[ 1 - \frac{1}{2} \sum_{\mina }^{\minasp}  1 - \erf \left(\frac{\sqrt{\Nmin}(\val - \meanmin)}{\sqrt{2 \sigma^2}} \right)  \bigg] \df\val \\
&= \frac{1}{2}\ln\left(\ \frac{2 \pi \sigma^2}{\Nmax} \right) \left( 1 - \sum_{\mina}^{\minasp}  \frac{1}{2}\left[ 1 - \erf \left(\frac{\sqrt{\Nmin}(\meanmax - \meanmin)}{\sqrt{2 \sigma^2 (\frac{1}{\Nmax}+ \frac{1}{\Nmin})}} \right) \right] \right)\\
&= \frac{1}{2}\ln\left(\ \frac{2 \pi \sigma^2}{\Nmax} \right) \left( 1 - \sum_{\mina }^{\minasp}  \frac{1}{2}\erfc \left[\frac{\sqrt{\Nmin}(\meanmax - \meanmin)}{\sqrt{2 \sigma^2 (\frac{1}{\Nmax}+ \frac{1}{\Nmin})}} \right] \right)\\
\end{split}
\end{equation}

Next, we separate the second term in two parts $T_{2,1}$ and $T_{2,2}$, first :

\begin{equation}\label{SiSb3_bis}
\begin{split}
T_{2,1} = &\int_{-\infty}^{+\infty} \frac{ \sqrt{\Nmax} e^{- \frac{\Nmax(\val - \meanmax)^2}{2 \sigma^2} }}{\sqrt{2\pi \sigma^2}}  \frac{\Nmax(\val - \meanmax)^2}{2 \sigma^2} \df\val  = \frac{1}{2}.
\end{split}
\end{equation}

Then, we integrate by parts the remaining term $T_{2,2}$ to obtain:
\begin{equation}\label{SiSb3}
\begin{split}
 &T_{2,2}  = -\sum_{\mina }^{\minasp} \int_{-\infty}^\infty  \frac{\Nmax^{3/2}(\val - \meanmax)^2}{4 \sigma^2} \frac{ e^{- \frac{\Nmax(\val - \meanmax)^2}{2 \sigma^2} }}{\sqrt{2\pi \sigma^2}}   \left[ 1 - \erf \left(\frac{\sqrt{\Nmin}(\val - \meanmin)}{\sqrt{2 \sigma^2}} \right) \right]\df\theta\\  & = -\sum_{\mina}^{\minasp} \int_{-\infty}^\infty \frac{1}{4}  \frac{ \sqrt{\Nmax} e^{- \frac{\Nmax(\val - \meanmax)^2}{2 \sigma^2} }}{\sqrt{2\pi \sigma^2}}  \left[ 1 - \erf \left(\frac{\sqrt{\Nmin}(\val - \meanmin)}{\sqrt{2 \sigma^2}} \right) \right]\\ &\hspace{3cm}
- \frac{(\val-\meanmax)}{2}  \frac{ \sqrt{\Nmax \Nmin} e^{- \frac{\Nmax(\val - \meanmax)^2}{2 \sigma^2}  - \frac{\Nmin(\val - \meanmin)^2}{2 \sigma^2}}}{2\pi \sigma^2}  \df\theta \\
     &=  -\sum_{ \armmax}^{\minasp} \frac{1}{4} \erfc \left(\frac{\meanmax-\meanmin}{\sqrt{2\sigma^2(\frac{1}{\Nmax}+ \frac{1}{\Nmin})}} \right)  - 
   \frac{(\meanmax - \meanmin) \sigma^2}{2 \Nmax \sqrt{2 \pi}(\frac{\sigma^2}{\Nmax} + \frac{\sigma^2}{\Nmin})^{3/2}} e^{-\frac{ (\meanmax - \meanmin)^2}{2\sigma^2(\frac{1}{\Nmax} + \frac{1}{\Nmin})}}, 
\end{split}
\end{equation}
where we also rely on the identity of \cref{Sierfint}. 

Combining \cref{SiSb2,SiSb3_bis,SiSb3} leads to the analytical expression of the body component.

\begin{equation}\label{Sbodygauss1}
\begin{split}
\Sc &= \frac{1}{2}\ln ( \frac{2\pi \varp e}{\Nmax}) \left[ 1 - \frac{1}{2} \sum_{\mina}^{\minasp}  \erfc \left(  \frac{ \sqrt{\Nmin \Nmax} (\meanmax - \meanmin)}{\sqrt{2\varp (\Nmin + \Nmax)}} \right) \right] \\  & \hspace{4cm}- \sum_{\mina }^{\minasp}  \frac{ \sqrt{\Nmax} \Nmin^{3/2}  (\meanmax - \meanmin) }{2 \sigma  \sqrt{2 \pi }  (\Nmax + \Nmin)^{3/2}} e^{-\frac{ \Nmin \Nmax (\meanmax - \meanmin)^2}{2\varp (\Nmin + \Nmax)} }.\\
\end{split}
\end{equation}

To finally get an asymptotic and simplified expression of the body component, we neglect the second term. Then, since  $\teqmin \underset{\Nmax \to \infty}{\longrightarrow} \meanmax$ and $\Nmin  \ll \Nmax$ asymptotically, we approximate the first term as:

\begin{equation}\label{Sbodygauss2}
\begin{split}
\Scapp &= \frac{1}{2}\ln ( \frac{2\pi \varp e}{\Nmax}) \left[ 1 - \frac{1}{2} \sum_{\mina}^{\minasp} \erfc \left(  \frac{ \sqrt{\Nmin } (\teqmin - \meanmin)}{\sqrt{2\varp}} \right) \right]. \\
\end{split}
\end{equation}

This last approximation will enable to provide an analytically tractable gradient without altering the asymptotic behavior expected at large times for the entropy measure.

\subsection{ Closed form and asymptotic expression for the tail's entropy}\label{reftailgauss}

The contribution from the tail can be derived exactly and reads:
\begin{equation}\label{SiSatil1}
\begin{split}
\Stail &=  \sum_{\mina}^{\minasp}\int_{\teqmin}^{\infty}  \frac{ \sqrt{\Nmin} e^{- \frac{\Nmin(\val - \meanmin)^2}{2 \sigma^2} }}{\sqrt{2\pi \sigma^2}} \bigg[ \frac{1}{2}\ln(\frac{2\pi\sigma^2}{\Nmin}) + \frac{\Nmin(\val - \meanmin)^2}{2 \sigma^2}  \bigg] \df \theta\\ 
&= \sum_{\mina }^{\minasp} \frac{1}{4}\ln \left( \frac{2\pi \sigma^2e}{\Nmin} \right) \erfc \left( \frac{\sqrt{\Nmin} (\teqmin-\meanmin)}{\sqrt{2\sigma^2}} \right)\\ &\hspace{5cm} + \frac{\sqrt{\Nmin} (\teqmin-\meanmin)}{2\sqrt{2\pi \sigma^2}} e^{- \frac{\Nmin(\teqmin - \meanmin)^2}{2 \sigma^2} }.
\end{split}
\end{equation}

To get a simplified analytical expression of the tail component, we only keep the second term since it dominates the others asymptotically,
\begin{equation}\label{SiSatil2}
\begin{split}
\Stailapp &=  \sum_{\mina}^{\minasp} \frac{\sqrt{\Nmin} (\teqmin-\meanmin)}{2\sqrt{2\pi \sigma^2}} e^{- \frac{\Nmin(\teqmin - \meanmin)^2}{2 \sigma^2} }.\end{split}
\end{equation}

Taken altogether, \cref{Sbodygauss2,SiSatil2} lead to the desired simplified approximation of the entropy:
\begin{equation}
\Sappgrad = \Scapp + \Stailapp.
\end{equation}

\subsection{Derivation of the increment for the closed-form expression of entropy }\label{Sigradientgaussian}

Since \cref{Sbodygauss2,SiSatil2} exhibit simple closed-form expressions, it becomes possible to derive an explicit expression of its expected increment. Here, we again consider continuous Gaussian reward distributions.

We start by deriving the increment along the better empirical arm, $\Delta_{\armmax}\Sappgrad$. The posterior of the reward obtained at time $t+ 1$ is approximated as a Gaussian of variance $\sigma^2$ and centred around $\meanmax$, leading to:
\begin{equation}\label{Sigradmax}
\begin{split}
\Delta_{\armmax}\Sappgrad &=  \int_{-\infty}^{\infty} \frac{e^{-\frac{\mu^2}{2 \sigma^2}}}{\sqrt{2\pi \sigma^2}} \bigg[ \Sappgrad( \meanmax + \frac{\mu}{\Nmax + 1}, \Nmax+ 1, ...) - \Sappgrad(\meanmax, \Nmax, ...)  \bigg]\df\mu.\\
\end{split}
\end{equation}

where the dots runs over all the worse empirical arms variables remaining constant when the best empirical arm is drawn at time $t+1$. 


For the sake of simplicity, we neglect the variations of all the subdominant terms inside all $\teqmin$ meaning we approximate them as $\teqmin( \meanmax + \frac{\mu}{\Nmax + 1}, \Nmax+ 1, \meanmin, \Nmin) \approx \teqmin(\meanmax, \Nmax, \meanmin, \Nmin) + \frac{\mu}{\Nmax + 1} $, after observing a reward $\mu$ when pulling the arm $\armmax$ for the ($\Nmax\!+\!1$)th time.

By use of the identity \cref{Sierfint}, the gradient of the body component $\Delta_{\armmax} \Scapp$ can be rewritten as: 
\begin{equation}\label{Sbodygrad1}
\begin{split}
\Delta_{\armmax} \Scapp   &= \frac{1}{2}\ln ( \frac{2\pi \varp e}{\Nmax + 1}) \left[ 1 - \frac{1}{2}  \sum_{\mina}^{\minasp} \erfc \left(  \frac{ \sqrt{\Nmin } (\teqmin - \meanmin)}{\sqrt{2\varp} \sqrt{1 + \frac{\Nmin}{(\Nmax + 1)^2}}} \right) \right] \\ &\hspace{3cm}- \frac{1}{2}\ln ( \frac{2\pi \varp e}{\Nmax }) \left[ 1 - \frac{1}{2}  \sum_{\mina }^{\minasp} \erfc \left(  \frac{ \sqrt{\Nmin } (\teqmin - \meanmin)}{\sqrt{2\varp}} \right) \right]. \\
\end{split}
\end{equation}
The increment of the tail component along the better empirical arm can be calculated as:
\begin{equation}\label{Stailgrad1}
\begin{split}
\Delta_{\armmax} \Stailapp  &=  \sum_{\mina }^{\minasp}   \int_{-\infty}^{\infty} \frac{e^{-\frac{\mu^2}{2 \sigma^2}}}{\sqrt{2\pi \sigma^2}} \frac{\sqrt{\Nmin} ( \frac{\mu}{\Nmax + 1} + \teqmin-\meanmin)}{2\sqrt{2\pi \sigma^2}} e^{- \frac{\Nmin(\teqmin + \frac{\mu}{\Nmax + 1} - \meanmin)^2}{2 \sigma^2} } \df \mu \\& \hspace{5cm}   -\sum_{\mina}^{\minasp} \frac{\sqrt{\Nmin} (\teqmin-\meanmin)}{2\sqrt{2\pi \sigma^2}} e^{- \frac{\Nmin(\teqmin - \meanmin)^2}{2 \sigma^2} } \\
&=  \sum_{\mina }^{\minasp}   e^{-\Nmin\frac{(\teqmin - \meanmin)^2}{2 \sigma^2 \left(1 + \frac{\Nmin}{(1 + \Nmax)^2} \right) }}  \sqrt{\frac{\Nmin}{8 \pi \sigma^2 }} \frac{ \left( \teqmin - \meanmin \right)}{ (1 + \frac{\Nmin}{(\Nmax + 1)^2})^{3/2}}\\  &\hspace{5cm}-  \sum_{\mina }^{\minasp} \frac{\sqrt{\Nmin} (\teqmin-\meanmin)}{2\sqrt{2\pi \sigma^2}} e^{- \frac{\Nmin(\teqmin - \meanmin)^2}{2 \sigma^2} }. \\
\end{split}
\end{equation}

Next, we consider the increment evaluation along a given worse empirical arm denoted by $k$,
\begin{equation}\label{Sigradmin}
\begin{split}
\Delta_{k} \Sappgrad &=  \int_{-\infty}^{\infty} \frac{e^{-\frac{\mu^2}{2 \sigma^2}}}{\sqrt{2\pi \sigma^2}} \bigg[ \Sappgrad( ..., \meana{k} + \frac{\mu}{\Nk + 1}, \Nk + 1,...) - \Sappgrad(..., \meana{k}, \Nk,...) \bigg]\df\mu.\\
\end{split}
\end{equation}

We here also neglect the variations of the subdominant term inside $\teqmin$. We start by considering the increment of the body component:
\begin{equation}\label{Sbodygrad2}
\begin{split}
\Delta_{k} \Scapp  &= \frac{1}{2}\ln ( \frac{2\pi \varp e}{\Nmax}) \left[ 1 - \frac{1}{2}\erfc \left(  \frac{  (\Nk+1)(\teqk - \meana{k}) }{\sqrt{2 \varp  ( \Nk + 2 ) }}  \right) \right] \\& \hspace{3cm}- \frac{1}{2}\ln ( \frac{2\pi \varp e}{\Nmax }) \left[ 1 - \frac{1}{2} \erfc \left(  \frac{ \sqrt{\Nk } (\teqk - \meana{k})}{\sqrt{2\varp}} \right) \right]. \\
\end{split}
\end{equation}

Of note, all other terms in the sum independent of index $k$ remain constant, showing no increment.
Finally, we consider the associated tail component of the increment along $k$:

\begin{equation}\label{Stailgrad2}
\begin{split}
\Delta_{k} \Stailapp &=   \int_{-\infty}^{\infty} \frac{e^{-\frac{\mu^2}{2 \sigma^2}}}{\sqrt{2\pi \sigma^2}} \frac{\sqrt{\Nk + 1} ( \frac{\mu}{\Nk + 1} + \teqk-\meank)}{2\sqrt{2\pi \sigma^2}} e^{- \frac{(\Nk + 1)(\teqk + \frac{\mu}{\Nk + 1} - \meank)^2}{2 \sigma^2} }\df \mu \\ &\hspace{6cm}-\frac{\sqrt{\Nk} (\teqk-\meank)}{2\sqrt{2\pi \sigma^2}} e^{- \frac{\Nk(\teqk - \meank)^2}{2 \sigma^2} }\\
&= e^{-\frac{(\Nk+1)^2}{(\Nk + 2)}\frac{(\teqk - \meank)^2}{2 \sigma^2  }} \frac{ ( 1 + \Nk)^2 \left( \teqk - \meank \right)}{\sqrt{8 \pi \sigma^2}  (2 +\Nk)^{3/2}}\\ &\hspace{6cm} -  \frac{\sqrt{\Nk} (\teqk-\meank)}{2\sqrt{2\pi \sigma^2}} e^{- \frac{\Nk(\teqk - \meank)^2}{2 \sigma^2} }\\  
\end{split}
\end{equation}

As for the body increment, all other terms in the sum independent of index $k$ remain constant, showing no increment.

Taken altogether, \cref{Sbodygrad1,Stailgrad1,Sbodygrad2,Stailgrad2} lead to the final analytical expression of the increment:

\begin{equation}\label{deltaheavy}
\begin{split}
\Delta_{\armmax, k} = & \frac{1}{2}\ln( \frac{\Nmax}{\Nmax+1}) -\frac{1}{4}\ln( \frac{2 \pi \sigma^2 e}{\Nmax+1}) \sum_{\mina}^{\minasp}   \erfc \left[  \frac{  \sqrt{\Nmin}(\teqmin - \meanmin) }{\sqrt{2 \varp  ( 1 + \frac{\Nmin}{(\Nmax + 1)^2}) } }    \right] \\& \hspace{1cm} +  \frac{1}{4}\ln( \frac{2 \pi \sigma^2}{\Nmax}) \sum_{\mina}^{\minasp} \erfc \left[ \sqrt{\Nmin}\frac{\teqmin-\meanmin}{\sqrt{2\sigma^2}}  \right]\\ & \hspace{1cm}+\sum_{\mina }^{\minasp}   e^{-\Nmin\frac{(\teqmin - \meanmin)^2}{2 \sigma^2 (1 + \frac{\Nmin}{(1 + \Nmax)^2}) }}  \sqrt{\frac{\Nmin}{8 \pi \sigma^2 }} \frac{ \left( \teqmin - \meanmin \right)}{ (1 + \frac{\Nmin}{(\Nmax + 1)^2})^{3/2}}  \\
&\hspace{1cm}-  \sum_{\mina }^{\minasp} \frac{\sqrt{\Nmin} (\teqmin-\meanmin)}{\sqrt{8\pi \sigma^2}} e^{- \frac{\Nmin(\teqmin - \meanmin)^2}{2 \sigma^2} }\\
&+\frac{1}{4}\ln( \frac{2 \pi \sigma^2 e}{\Nmax}) \left[ \erfc \left(  \frac{ (\Nk + 1)(\teqk - \meank) }{\sqrt{2 \varp  ( 2 + \Nk ) } } \right) -\erfc \left(\sqrt{\Nk}\frac{\teqk -\meank}{\sqrt{2\sigma^2}} \right) \right]\\
&-e^{-\frac{(\Nk+1)^2}{(\Nk + 2)}\frac{(\teqk - \meank)^2}{2 \sigma^2  }} \frac{ ( 1 + \Nk)^2 \left( \teqk - \meank \right)}{\sqrt{8 \pi \sigma^2}  (2 +\Nk)^{3/2}} +\frac{\sqrt{\Nk} (\teqk-\meank)}{2\sqrt{2\pi \sigma^2}} e^{- \frac{\Nk(\teqk - \meank)^2}{2 \sigma^2} }\\  
\end{split}
 \end{equation}

To obtain a simplified expression, we expand to the first order each component of the different components of \cref{deltaheavy} denoted $T_1$, $T_2$, $T_3$, $T_4$. The former is given by:
\begin{equation}
\begin{split}
T_1 &=  -\frac{1}{4}\ln( \frac{2 \pi \sigma^2 e}{\Nmax+1}) \sum_{\mina }^{\minasp}  \erfc \left[  \frac{  \sqrt{\Nmin}(\teqmin - \meanmin) }{\sqrt{2 \varp  ( 1 + \frac{\Nmin}{(\Nmax + 1)^2}) } }    \right] \\ &\hspace{5cm}+  \frac{1}{4}\ln( \frac{2 \pi \sigma^2 e}{\Nmax}) \sum_{\mina \neq \armmax}^{K} \erfc \left[ \sqrt{\Nmin}\frac{\teqmin-\meanmin}{\sqrt{2\sigma^2}}  \right] \\ &\approx  \sum_{\mina }^{\minasp} \frac{1}{4 \Nmax}\erfc \left(  \frac{ \sqrt{\Nmin} (\teqmin - \meanmin)}{\sqrt{2\varp}} \right) \\ &\hspace{4cm} - \frac{1}{4}\ln ( \frac{2\pi \varp e}{\Nmax})\frac{\Nmin}{\Nmax^2} \frac{ \sqrt{\Nmin} (\teqmin - \meanmin)}{\sqrt{2 \pi \varp}} e^{-\frac{\Nmin  (\teqmin - \meanmin)^2}{2\varp  }}.\\
\end{split}
\end{equation}

Next  we consider the second component, which reads:

\begin{equation}
\begin{split}
T_2 &= \sum_{\mina }^{\minasp}    e^{-\Nmin\frac{(\teqmin - \meanmin)^2}{2 \sigma^2 (1 + \frac{\Nmin}{(1 + \Nmax)^2}) }}  \sqrt{\frac{\Nmin}{8 \pi \sigma^2 }} \frac{ \left( \teqmin - \meanmin \right)}{ (1 + \frac{\Nmin}{(\Nmax + 1)^2})^{3/2}}  \\&
\hspace{6cm}- \sum_{\mina}^{\minasp}   \frac{\sqrt{\Nmin} (\teqmin-\meanmin)}{2\sqrt{2\pi \sigma^2}} e^{- \frac{\Nmin(\teqmin - \meanmin)^2}{2 \sigma^2} }\\
&\approx \sum_{\mina}^{\minasp} e^{-\Nmin\frac{(\teqmin - \meanmin)^2}{2 \sigma^2  }}  \sqrt{\frac{\Nmin}{8 \pi \sigma^2 }}\left( \teqmin - \meanmin \right) \left[  -\frac{3}{2} \frac{\Nmin}{\Nmax^2} + \frac{(\teqmin - \meanmin)^2}{2 \sigma^2  } \frac{\Nmin^2}{\Nmax^2} \right] .\\
\end{split}
\end{equation}

Next we consider the third term, denoted $T_3$, which reads:

\begin{equation}
\begin{split}
T_3 & = \frac{1}{4}\ln( \frac{2 \pi \sigma^2 e}{\Nmax}) \left[ \erfc \left(  \frac{ (\Nk + 1)(\teqk - \meank) }{\sqrt{2 \varp  ( 2 + \Nk ) } } \right) -\erfc \left(\sqrt{\Nk}\frac{\teqk -\meank}{\sqrt{2\sigma^2}} \right) \right]\\
& \approx -\frac{1}{4}\ln( \frac{2 \pi \sigma^2 e}{\Nmax}) \frac{1}{\Nk^2} \left(  \frac{ \sqrt{\Nk } (\teqk - \meank)}{\sqrt{2 \pi\varp}} \right) e^{-\frac{\Nk  (\teqk - \meank)^2}{2\varp  }}
\end{split}
\end{equation}

Finally, the last term $T_4$ reads

\begin{equation}
\begin{split}
T_4 &= -e^{-\frac{(\Nk+1)^2}{(\Nk + 2)}\frac{(\teqk - \meank)^2}{2 \sigma^2  }} \frac{ ( 1 + \Nk)^2 \left( \teqk - \meank \right)}{\sqrt{8 \pi \sigma^2}  (2 +\Nk)^{3/2}} +\frac{\sqrt{\Nk} (\teqk-\meank)}{2\sqrt{2\pi \sigma^2}} e^{- \frac{\Nk(\teq - \meank)^2}{2 \sigma^2} }\\
 &\approx e^{-\Nk\frac{(\teqk - \meank)^2}{2 \sigma^2  }}  \sqrt{\frac{\Nk}{8 \pi \sigma^2 }} \left( \teqk - \meank \right) \left[ \frac{1}{\Nk} + \frac{1}{\Nk}\frac{(\teqk - \meank)^2}{2 \sigma^2  }   \right].
 \end{split}
\end{equation}

Taken altogether, we finally obtain the following simplified increment :

\begin{equation}
\begin{split}
&\overset{\sim}{\Delta}_{\armmax, k} =  \frac{1}{2}\ln ( \frac{\Nmax}{\Nmax + 1}) +  \frac{1}{2 \Nmax} \sum_{\mina}^{\minasp} \frac{1}{2}\erfc \left(  \frac{ \sqrt{\Nmin} (\teqmin - \meana{\mina})}{\sqrt{2\varp}} \right) \\ & \hspace{-0.2cm}+  \sum_{\mina}^{\minasp} \frac{\Nmin^{3/2} \left( \teqmin - \meana{\mina} \right)}{\sqrt{2 \pi \sigma^2} \Nmax^2 }e^{-\Nmin\frac{(\teqmin -\meana{\mina})^2}{2 \sigma^2  }} \left[ \frac{1}{4}\ln ( \frac{\Nmax}{2\pi \varp e})  -\frac{3}{4}+ \frac{\Nmin(\teqmin - \meana{\mina})^2}{4 \sigma^2  } \right]\\
&+  \frac{ \teqk - \meank}{ \sqrt{2 \pi \sigma^2 \Nk }}  e^{-\Nk\frac{(\teqk - \meank)^2}{2 \sigma^2  }}  \ \left[ \frac{1}{4 \Nk}\ln( \frac{\Nmax}{2 \pi \sigma^2 e})   + \frac{1}{2} + \frac{(\teqk - \meank)^2}{4 \sigma^2 } \right]\\
\end{split}
\end{equation}

By noticing that the sum of second term should account for the tail contribution along the body increment, it shouldn't be allowed to be superior to one. Then we assume to bound it by taking the minimum compared to $1-1/K$.

\subsection{Final expression for the increment comparison}\label{SifinalformdeltaSi}

Taken altogether, it leads to expression used for AIM for multiple gaussian arms:

\begin{equation}
\begin{split}
&\Delta_{\armmax, k} =  \frac{1}{2}\ln ( \frac{\Nmax}{\Nmax + 1}) +  \frac{1}{2 \Nmax}  \min \left( \sum_{\mina}^{\minasp} \frac{1}{2}\erfc \left(  \frac{ \sqrt{\Nmin} (\teqmin - \meana{\mina})}{\sqrt{2\varp}} \right), 1- \frac{1}{K} \right)\\ & \hspace{-0.2cm}+  \sum_{\mina}^{\minasp} \frac{\Nmin^{3/2} \left( \teqmin - \meana{\mina} \right)}{\sqrt{2 \pi \sigma^2} \Nmax^2 }e^{-\Nmin\frac{(\teqmin -\meana{\mina})^2}{2 \sigma^2  }} \left[ \frac{1}{4}\ln ( \frac{\Nmax}{2\pi \varp e})  -\frac{3}{4}+ \frac{\Nmin(\teqmin - \meana{\mina})^2}{4 \sigma^2  } \right]\\
&+  \frac{ \teqk - \meank}{ \sqrt{2 \pi \sigma^2 \Nk }}  e^{-\Nk\frac{(\teqk - \meank)^2}{2 \sigma^2  }}  \ \left[ \frac{1}{4 \Nk}\ln( \frac{\Nmax}{2 \pi \sigma^2 e})   + \frac{1}{2} + \frac{(\teqk - \meank)^2}{4 \sigma^2 } \right]\\
\end{split}
\end{equation}

with 

\begin{equation}
\teqi = \meanmax + \frac{\Ni(\meanmax -\meana{i})}{\Nmax-\Ni}+  \sqrt{ \frac{\Nmax \Ni( \meanmax-\meana{i})^2}{(\Nmax - \Ni)^2}  + \frac{\sigma^2 \ln \left( \frac{\Nmax}{\Ni} \right) }{\Nmax- \Ni}  } .
\end{equation}

\section{Proof of \cref{thm:main}}\label{app:proof}

This section provides the complete proof of \cref{thm:main}. More precisely, it proves the more refined \cref{thm:app} below.

\begin{theorem}\label{thm:app}
    For any multi-armed bandits with Gaussian rewards of variance $\sigma^2$ and mean vector $\pmb{\mu}\in\bR^K$, for any $\varepsilon\in(0,\frac{1}{2})$, there exists a constant $C(\pmb{\mu},\varepsilon)\in \bR$ depending solely on $\pmb{\mu}$ and $\varepsilon$ such that for any $T\in\bN$
    \begin{equation*}
        \RegretT \leq \sum_{k, \mu_k<\mu^*}\Bigg[\frac{2\varp\ \ln T }{(1-\varepsilon)(\mu^*-\mu_k)} +  \frac{2\varp\ \ln\ln T}{(1-\varepsilon)(\mu^*-\mu_k)}\Bigg]+C(\pmb{\mu},\varepsilon).
    \end{equation*}
\end{theorem}

%
%

\begin{proof}
   We denote in the whole proof $\maxarms=\{k\in[K]\mid \mu_k=\mu^*\}$. For $\Delta_k=\mu^*-\mu_k$, the regret can then be written as 
    \begin{align*}
         \RegretT&=\sum_{k,\Delta_k>0}\Delta_k \E{N_k(T)}.
    \end{align*}
   We decompose this expectation in $4$ terms as follows
    \begin{align*}
        \E{N_k(T)} &\leq \sum_{t=1}^T \Prob\left(\forall i \in\maxarms, N_i(t)\leq\sqrt{t}\right) 
        + \sum_{t=1}^T \Prob\left(\meana{k}(t)\geq \mean^*-\sqrt{\frac{6\varp\ln t}{\sqrt{t}}}, \Armt=k\right) 
        \\& \hspace{1cm}\phantom{=}+ \sum_{t=1}^T \Prob\left(\exists i\in\maxarms,\meana{i}(t)\leq \mean_i-\sqrt{\frac{6\varp\ln t}{N_i(t)}}\right)
        + \sum_{t=1}^T \Prob\left(\cE_k(t)\right),
    \end{align*}
    where
    \begin{equation*}
    \cE_k(t) \coloneqq \left\{\exists i\in\maxarms, N_i(t)\geq \sqrt{t} \text{ and } \meana{i}(t)\geq\mean^*-\sqrt{\frac{6\varp\ln t}{N_i(t)}}\geq \meana{k}(t), \Armt=k\right\}.
    \end{equation*}
    This inequality comes simply by noticing the event $\{\Armt=k\}$ is included in the union of the $4$ other events. \cref{lemma:tb,lemma:surestim,lemma:underestim}
allow to respectively bound the first, second and third sums by a constant $C(\pmb{\mu})$ depending solely on $\pmb{\mu}$, so that
\begin{equation*}
    \E{N_k(T)} \leq \sum_{t=1}^T \Prob\left(\cE_k(t)\right)  + C(\pmb{\mu}).
\end{equation*}
Thanks to \cref{lemma:inclusion}, there exist constants $t(\pmb{\mu}),n(\pmb{\mu})$ depending solely on $K$ and $\Delta_k$ such that
\begin{align*}
    \sum_{t=1}^T \Prob\left(\cE_k(t)\right) \leq t(\pmb{\mu})+ \sum_{t=1}^T \Prob\left(\cG_1(t)\right) + \Prob\left(\cG_2(t)\right),
\end{align*}
where
\begin{gather*}
    \cG_1(t) = \{\mean_k-\meana{k}(t)\leq-\varepsilon\Delta_k,\Armt=k\},\\
   \cG_2(t) = \{N_k(t) \leq \frac{2\varp}{(1-2\varepsilon)^2\Delta_k^2} \big(\ln t +  \ln\ln t\big)+n(\pmb{\mu}),\Armt=k\}.
\end{gather*}

Now, we bound individually the sum corresponding to each of these $2$ events. 
The first one can be bounded using Hoeffding's inequality. Indeed, for independent random variables $Z_k(n)\sim \cN(\mean_k,\varp)$, it reads as:
\begin{align*}
    \sum_{t=1}^T\Prob\left(\mean_k-\meana{k}(t)\leq-\varepsilon\Delta_k,\Armt=k\right) & \leq \sum_{n=1}^T \Prob\left(\frac{1}{n}\sum_{i=1}^n Z_k(i)-\mean_k \geq \varepsilon\Delta_k\right)\\
    &\leq \sum_{n=1}^T e^{-\frac{n\varepsilon^2\Delta_k^2}{2\varp}}\\
    &\leq \frac{1}{e^{\frac{\varepsilon^2\Delta_k^2}{2\varp}-1}}.
    \end{align*}
    The bound of the second term is bounded as
    \begin{align*}
        \sum_{t=1}^T\Prob(\cE_k(t))&\leq \E{\sum_{t=1}^T \Indicator\{N_k(t)\leq\frac{2\varp}{(1-2\varepsilon)^2\Delta_k^2} \big(\ln t +  \ln\ln t\big)+ n(\pmb{\mu}), \Armt=k\}}\\
        &\leq \frac{2\varp}{(1-2\varepsilon)^2\Delta_k^2} \big(\ln T +  \ln \ln T\big)+n(\pmb{\mu})+1.
    \end{align*}
    Wrapping up everything finally yields that for some constant $C(\pmb{\mu},\varepsilon)$ depending solely on $\pmb{\mu}, \varepsilon$,
    \begin{equation*}
        \RegretT \leq \frac{2\varp}{(1-2\varepsilon)^2\Delta_k} \left(\ln T +  \ln \ln T\right)+C(\pmb{\mu},\varepsilon).
    \end{equation*}
    This concludes the proof of \cref{thm:app} with the reparameterization $\varepsilon\gets 1-(1-2\varepsilon)^2$.
\end{proof}

\subsection{Auxiliary Lemmas}\label{app:auxlemmas}
Similarly to the proof of Thompson sampling, the first part of the proof shows that the optimal arm is at least pulled a polynomial number of times with high probability. We recall that we denote in this whole section $\maxarms = \arg\max_k \mu_k$.
\begin{lemma}\label{lemma:tb}
    There exists a constant $C_0(\pmb{\mu})$ depending solely on the mean vector $\pmb{\mu}$ such that
    \begin{equation*}
        \sum_{t=1}^{\infty}\Prob(\forall i \in \maxarms, N_i(t)\leq \sqrt{t}) \leq C_0(\pmb{\mu}).
    \end{equation*}
\end{lemma}
\begin{proof}
    Let $t_0(\pmb{\mu})$ be a large constant that depends solely on $\pmb{\mu}$. In the remaining of the proof, we 
    assume at some points that $t_0(\pmb{\mu})$ is chosen large enough (but only larger than a threshold depending on $\pmb{\mu}$) such that some inequalities hold. We also assume in the following, without loss of generality, that $\mu_1=\mu^*$, i.e., $1\in\maxarms$.
    
    Assume that for $t\geq t_0(\pmb{\mu})$, $N_i(t)\leq \sqrt{t}$ for all $i\in\maxarms$. 
    Let then $k$ be the most pulled arm at time $t$, i.e., $k\in\argmax_{j}N_j(t)$ (if multiple arms maximise the number of pulls, we select the one such that its last pull happened the earliest). Necessarily $N_k(t)\geq \frac{t}{K}$. We can choose $t_0(\pmb{\mu})$ large enough so that $\frac{t}{K}>\sqrt{t}$ and thus $\Delta_k>0$.
  Let $t'\leq t$ be the last time $k$ was pulled.  By design, $k$ also maximised the number of pulls then, so that $k\in\argmax_j\meana{j}(t')$. Moreover, $N_i(t')\leq\sqrt{t}$ for all $i\in\maxarms$ and $N_k(t')\geq\frac{t}{K}-1$. For $t_0(\pmb{\mu})$ large enough, this yields $N_k(t')\geq N_i(t')$ for all $i\in\maxarms$ and $a_{t'}=k$. 
    The arm $k$ is thus pulled at time $t'$, in particular because $S_k\geq S_1$ (i.e., $\Delta_{k,1}S \leq 0$), where
    \begin{align*}
        S_k &= \frac{1}{2}\ln \left(1+\frac{1}{N_k(t')} \right) - \frac{1}{2 N_k(t')} \min \left( \frac{1}{2}\sum_{i \neq k} \erfc \left(  \frac{ \sqrt{\Ni(t')} (\teqi - \meanah{i})}{\sqrt{2\varp}} \right), \, 1-\frac{1}{K}\right),\\
        S_1& =  \sum_{i \neq k}  \sqrt{\frac{\Ni(t')}{2 \pi \sigma^2 }}\left( \teqi - \meanah{i} \right) e^{-\Ni(t')\frac{(\teqi - \meanah{i})^2}{2 \sigma^2  }} \left[\frac{1}{4}\ln ( \frac{N_k(t')}{2\pi \varp e})\frac{\Ni(t')}{N_k^2(t')}  -\frac{3}{4} \frac{\Ni(t')}{N_k^2(t')} + \frac{(\teqi - \meanah{i})^2}{4 \sigma^2  } \frac{\Ni^2(t')}{N_k^2(t')}  \right]\\
&+ e^{-N_1(t')\frac{(\teqi[1] - \meanah{1})^2}{2 \sigma^2  }}  \sqrt{\frac{N_1(t')}{2 \pi \sigma^2 }} \left( \teqi[1] - \meanah{1} \right) \left[ \frac{1}{4}\ln( \frac{N_k(t')}{2 \pi \sigma^2 e}) \frac{1}{N_1^2(t')}  + \frac{1}{2N_1(t')} + \frac{1}{N_1(t')}\frac{(\teqi[1] - \meanah{1})^2}{4 \sigma^2  } \right].
    \end{align*}
    To simplify, note that $S_k\leq \frac{1}{2N_k(t')}$. Moreover since $N_k(t')\geq\frac{t}{K}-1\geq 2\pi e^4\varp$ for a large enough choice of $t_0(\pmb{\mu})$, $S_1$ can be easily lower bounded as
    \begin{equation*}
        S_1\geq \frac{1}{2}\frac{(\teqi[1] - \meana1)}{\sqrt{2\pi\varp N_1(t')}}e^{-\frac{N_1(t')(\teqi[1]-\meana1)^2}{2\varp}}.
    \end{equation*}
So we finally have the following inequality at time $t'$:
\begin{equation}\label{eq:tb1}
\frac{1}{N_2(t')} \geq \frac{(\teq - \meana1)}{\sqrt{2\pi\varp N_1(t')}}e^{-\frac{N_1(t')(\teq-\meana1)^2}{2\varp}}.
\end{equation}
Recall that $N_2(t')\geq \frac{t}{K}-1$, so that \cref{eq:tb1} can be rewritten as
\begin{equation}
    N_1(t')\geq (\frac{t}{K}-1) \frac{\tilde{x}}{\sqrt{\pi}}e^{-\tilde{x}^2},
\end{equation}
where $\tilde{x}=\frac{\sqrt{N_1(t')}(\teq-\meana1)}{\sqrt{2\varp}}$. In the following, we will show that $\tilde{x}\in[\tilde{x}_{\min}, \tilde{x}_{\max}]\subset \mathbb{R}_+$. By analysing the variations of $x\mapsto x e^{-x^2}$, this will imply that 
\begin{equation}\label{eq:xbounds}
    N_1(t') \geq \frac{\frac{t}{K}-1}{\sqrt{\pi}} \min\{\tilde{x}_{\min}e^{-\tilde{x}_{\min}^2},\tilde{x}_{\max}e^{-\tilde{x}_{\max}^2} \}.
\end{equation}
For the lower bound, the definition of $\teqi[1]$ and the fact that $N_1(t')\geq 1$ directly implies that 
\begin{align*}
    \tilde{x}&\geq \sqrt{\frac{\ln(\frac{N_k(t')}{N_1(t')})}{2(\frac{N_k(t')}{N_1(t')}-1)}} = \Omega\left(\sqrt{\frac{\ln(t)}{t}}\right)\label{eq:xmin}.
\end{align*}
Moreover, by subadditivity of the square root:
\begin{align}
    \tilde{x} &\leq \sqrt{\frac{N_1(t')}{2\varp}}(\meana2-\meana1)\left(1+\frac{N_1(t')+\sqrt{N_1(t')N_k(t')}}{N_k(t')-N_1(t')}\right) + \sqrt{\frac{N_1(t')\ln(\frac{N_k(t')}{N_1(t')})}{2(N_k(t')-N_1(t'))}}\\
    & \leq \sqrt{\frac{N_1(t')}{2\varp}}(\meana{k}-\meana1)\left(1+Kt^{-\frac{1}{4}}\right) + \bigO{\frac{\sqrt{K\ln(t)}}{t^{\frac{1}{4}}}}.
\end{align}
Let us now consider the events, for $\Delta_{\min}=\min_{j, \Delta_j>0}\Delta_j$,
\begin{gather}
    \cH_{*}(t) \coloneqq \left\{\exists i \in \maxarms, \exists s\leq t, \meana{i}(s)-\mean_i\leq - \sqrt{\frac{2\varp(\ln(t)-\ln\ln(t))}{N_i(s)}} -\frac{\Delta_{\min}}{3} \right\}\label{eq:H1} ,\\
   \cH_{k}(t) \coloneqq \left\{\exists s\leq t, \frac{t}{K}-1\leq N_k(s) \leq t \text{ and } \meana{k}(s)-\mean_k\geq \frac{\Delta_k}{3} \right\} \label{eq:H2}.
\end{gather}
Assume in the following that $\neg\cH_*(t)\cap\neg\cH_k(t)$. This implies that
\begin{align}\label{eq:hDelta}
    \meana{k}-\meana1 \leq -\frac{\Delta_k}{3} + \sqrt{\frac{2\varp(\ln(t)-\ln\ln(t))}{N_1(s)}}.
\end{align}
In particular,
\begin{align*}
    \sqrt{\frac{N_1(t')}{2\varp}}(\meana2-\meana1) \leq \sqrt{\ln(t)-\ln\ln(t)},
\end{align*}
which implies that $\tilde{x} \leq \sqrt{\ln(t)-\ln\ln(t)}+\bigO{K\sqrt{\ln(t)}t^{-\frac{1}{4}}}$. Using the lower and upper bounds on $\tilde{x}$, we have thanks to \cref{eq:xbounds} that under $\neg\cH_*(t)\cap\neg\cH_k(t)$,
\begin{align*}
    N_1(t') = \Omega(\frac{\ln^{\frac{3}{2}}(t)}{K}).
\end{align*}
For a large enough choice of $t_0(\pmb{\mu})$, this last equality along with \cref{eq:hDelta} actually yield $\meana{k}-\meana1< 0$, which contradicts the beginning of the proof ($k$ being best empirical arm at time $t'$). By contradiction, we thus showed the following event inclusion for $t\geq t_0(\pmb{\mu})$:
\begin{equation}\label{eq:inclusion}
    \left\{\forall i\in\maxarms, N_i(t) \leq \sqrt{t}\right\} \subset \cH_*(t) \cup \cH_k(t).
\end{equation}
\cref{lemma:tb} then follows, thanks to \cref{lemma:Hs} below,
\begin{align*}
    \sum_{t=1}^{\infty}\Prob(\forall i\in\maxarms, N_i(t) \leq \sqrt{t}) \leq t_0(\pmb{\mu}) + \sum_{t=t_0(\pmb{\mu})+1}^{\infty}\Prob(\cH_*(t)) +\Prob(\cH_k(t)).
\end{align*}
\end{proof}

\begin{lemma}\label{lemma:Hs}
    For any $b\in(0,1)$ and the events $\cH_*(t), \cH_k(t)$ defined in \cref{eq:H1,eq:H2}, there exist constants $c_1$ and $c_2$ depending solely on $\pmb{\mu}$ such that
    \begin{gather*}
        \sum_{t=1}^{\infty} \Prob(\cH_*(t)) \leq c_1 \quad \mathrm{and}\quad
 \sum_{t=1}^{\infty} \Prob(\cH_k(t)) \leq  c_2 \quad \text{for any }k\not\in\maxarms.
\end{gather*}
\end{lemma}

\begin{proof}
The two bounds directly result from Hoeffding's inequality. Consider independent random variables $(Z_j(n))_{n\in\N,j\in[K]}$ where $Z_j(n)\sim\cN(\mean_j,\varp)$. Let us first bound the probability of $\cH_k(t)$, which is simpler.
\begin{align*}
    \Prob(\cH_k(t)) & \leq \sum_{n=\lceil t-t^b-1 \rceil}^{t} \Prob\left(\sum_{i=1}^n (Z_k(i)-\mu_k) \geq \frac{n\Delta_k}{3}\right)\\
    & \leq \sum_{n=\lceil \frac{t}{K}-1 \rceil}^{t} e^{-\frac{n\Delta_k^2}{18\varp}}\\
    &\leq \frac{e^{-\frac{\lceil \frac{t}{K}-1 \rceil\Delta_k^2}{18\varp}}}{1-e^{-\frac{\Delta_k^2}{18\varp}}}.
\end{align*}
The second inequality of \cref{lemma:Hs} then follows by noting that the last term is summable over $t$. For the second bound, we also have by Hoeffding's inequality
\begin{align*}
    \Prob(\cH_*(t)) & \leq \sum_{j\in\maxarms}\sum_{n=1}^{\infty} \Prob\left(\sum_{i=1}^n (Z_j(i)-\mean_j)\leq - \sqrt{2n\varp(\ln(t)-\ln\ln(t))} -\frac{n\Delta_{\min}}{3} \right)\\
    &\leq \sum_{j\in\maxarms}\sum_{n=1}^{\infty} \exp\left(-\ln(t)+\ln\ln(t) - \sqrt{2n\varp(\ln(t)-\ln\ln(t))}\frac{\Delta_{\min}}{3\varp} - \frac{n\Delta_{\min}^2}{18\varp}  \right)\\
    & \leq |\maxarms|\frac{\ln(t)}{t}\exp\left(-\sqrt{2(\ln(t)-\ln\ln(t))}\frac{\Delta_{\min}}{3\sqrt{\varp}}\right)\sum_{n=1}^{\infty}e^{- \frac{n\Delta_{\min}^2}{18\varp}}.
\end{align*}
The last sum is obviously finite. Moreover, $\sqrt{2(\ln(t)-\ln\ln(t))}=\omega(\ln\ln(t))$, so that $\exp\left(-\sqrt{2(\ln(t)-\ln\ln(t))}\frac{\Delta_{\min}}{3\sqrt{\varp}}\right)=\bigO{\frac{1}{\ln^{\alpha}(t)}}$ for any $\alpha>0$. By comparison with series of the form $\frac{1}{n\ln^{\alpha}(n)}$, the term $\frac{\ln(t)}{t}\exp\left(-\sqrt{2(\ln(t)-\ln\ln(t))}\frac{\Delta_2}{3\sqrt{\varp}}\right)$ is summable over $t$, which leads to the first bound of \cref{lemma:Hs}.
\end{proof}

\begin{lemma}\label{lemma:surestim}
    For any $k\not\in\maxarms$, there exists a constant $C_1(\pmb{\mu})$ depending solely on $\pmb{\mu}$ such that
    \begin{equation*}
        \sum_{t=1}^{\infty} \Prob\left(\meana{k}(t)\geq \mean^*-\sqrt{\frac{6\varp\ln t}{ \sqrt{t}}}, \Armt =k\right)\leq C_1(\pmb{\mu}).
    \end{equation*}
\end{lemma}

\begin{proof}
    A union bound on the sum yields for any $T\in\bN$
    \begin{multline*}
        \sum_{t=1}^{T} \Prob\left(\meana2(t)\geq \mean^*-\sqrt{\frac{6\varp\ln t}{\sqrt{t}}}, \Armt =k\right) 
        \\\hfill \leq \sum_{t=1}^{T} \sum_{n=0}^t\Prob\left(\meana{k}(t)\geq \mean^*-\sqrt{\frac{6\varp\ln t}{ \sqrt{t}}}, N_k(t)=n, N_k(t+1)=n+1\right)
       \\\hfill\leq \sum_{n=0}^T\sum_{t=n}^{T} \Prob\left({\color{vert}\underbrace{\color{black}\meana{k}(t)\geq \mean^*-\sqrt{\frac{6\varp\min_{s\geq n}\ln s}{\sqrt{s}}}, N_k(t)=n, N_k(t+1)=n+1}_{\coloneqq \cG_1(t,n)}}\right).
    \end{multline*}
    Now note that the $\cG_1(t,n)$ are disjoint for different $t$. In particular,
    \begin{align*}
        \sum_{t=n}^T \Prob(\cG_1(t,n)) & = \Prob\left(\exists t\in [n,T], \meana{k}(t)\geq \mean^*-\sqrt{\frac{6\varp \min_{s\geq n}\ln s}{s^b}}, N_2(t)=n\right).
    \end{align*}
For independent random variables $Z_k(n)\sim\cN(\mean_k,\varp)$, we have by independence of the $X_t$ and $\Armt$, and then by Hoeffding inequality:
\begin{align*}
    \sum_{t=1}^{T} \Prob\left(\meana{k}(t)\geq \mean^*-\sqrt{\frac{6\varp\ln t}{ \sqrt{t}}}, \Armt =k\right)&\leq 1+\sum_{n=1}^T \Prob\left(\frac{1}{n}\sum_{i=1}^n (Z_k(i)-\mu_k) \geq \Delta_k-\sqrt{\frac{6\varp\min_{s\geq n}\ln s}{ \sqrt{s}}}\right),\\
    & \leq 1+\sum_{n=1}^T \exp\left(-\frac{n\left(\Delta_k-\sqrt{\frac{6\varp\min_{s\geq n}\ln s}{\sqrt{s}}})\right)^2}{2\varp}\right).
\end{align*}
Obviously, this sum can be bounded for any $T\in\bN$ by a constant solely depending on $\Delta_k$.
\end{proof}

\begin{lemma}\label{lemma:underestim}
    For any $i \in[K]$, there exists a universal constant $C_2$ such that
    \begin{equation*}
        \sum_{t=0}^{\infty} \Prob\left(\meana{i}(t)\leq\mean_i-\sqrt{\frac{6\varp\ln t}{ N_i(t)}})\right)\leq C_2.
    \end{equation*}
\end{lemma}
\begin{proof}
    This is a direct consequence of \citet{garivier2013informational}, which states that for Gaussian rewards with variance $\varp$:
    \begin{equation*}
        \Prob(N_i(t) \frac{(\meana{i}(t)-\mean_i)^2}{2\varp} \geq (1+\alpha)\ln t)\leq 2\left\lceil\frac{\ln t}{\ln(1+\eta)} \right\rceil t^{-(1-\frac{\eta^2}{16})(1+\alpha)}
\quad\text{for any }t\in\bN^* \text{ and }\alpha,\eta>0 .   \end{equation*}
In particular with $\alpha=2=\eta$, this implies
\begin{align*}
    \Prob\left( \meana{i}(t) \leq \mean_i - \sqrt{\frac{6\varp\ln t}{ N_i(t)}} \right) \leq 2\frac{\ln (t) +1}{\ln(3)}t^{-{\frac{9}{4}}}.
\end{align*}
This term is obviously summable so that there exists a constant $C_2$ such that
\begin{equation*}
    \sum_{t=1}^{\infty} \Prob\left( \meana1(t) \leq \mean_1 - \sqrt{\frac{6\varp\ln t}{N_1(t)}} \right) \leq C_2.
\end{equation*}
\cref{lemma:underestim} directly follows by the inclusion of the considered events.
\end{proof}

For any $k\not\in\maxarms$, \cref{lemma:inclusion} below gives an event inclusion for the event $\cE_k(t)$ that we recall here,
\begin{equation*}
    \cE_k(t) \coloneqq \left\{\exists i\in\maxarms, N_i(t)\geq \sqrt{t} \text{ and } \meana{i}(t)\geq\mean^*-\sqrt{\frac{6\varp\ln t}{N_i(t)}}\geq \meana{k}(t), \Armt=k, \right\}.
\end{equation*}
\begin{lemma}\label{lemma:inclusion}
There exist constants $t(\pmb{\mu})$ and $n(\pmb{\mu})$ depending solely on $\pmb{\mu}$ such that for any $k\not\in\maxarms, t\geq t(\pmb{\mu})$ and $\varepsilon\in(0,\frac{1}{3})$,
\begin{align*}
     \cE_k(t) &\subset \{\mean_k-\meana{k}(t)\leq-\varepsilon\Delta_k,\Armt=k\}\cup\{N_k(t) \leq \frac{2\varp}{(1-2\varepsilon)^2\Delta_k^2} \big(\ln t +  \ln\ln t\big)+n(\pmb{\mu}),\Armt=k\}.
\end{align*}
\end{lemma}
\begin{proof}
    Assume in the following that $\cE_k(t)$ holds for some $t\geq t(\pmb{\mu})$. Let $i\in[K]$ be an arm maximising the empirical mean at time $t$. Necessarily $\meana{i}(t)\geq\mean^*-\sqrt{\frac{6\varp\ln t}{N_i(t)}}$. Moreover, $\Armt=k$ so that $i$ also maximises the number of pulls, in particular $N_i(t)\geq \frac{t}{K}$. Moreover, as we pull the arm $k$, $S_k\geq S_i$ where
    \begin{equation}\label{eq:diffS}
    \begin{aligned}
    S_i &= \frac{1}{2}\ln\left(1+\frac{1}{N_i(t)}\right) - \frac{1}{2 N_i(t)} \min\left( \frac{1}{2}\sum_{j \neq i} \erfc \left(  \frac{ \sqrt{N_j(t)} (\teqi[j] - \meanah{j})}{\sqrt{2\varp}} \right), \, 1-\frac{1}{K}\right),\\
        S_k& = g_k(t)Q_k(t) + \sum_{j \neq i}  g_j(t) P_j(t),
    \end{aligned}\end{equation}
    where for all $j\neq i$
    \begin{gather*}
       g_j(t)= \sqrt{\frac{N_j(t)}{2 \pi \sigma^2 }}\left( \teqi[j] - \meanah{j} \right) e^{-N_j(t)\frac{(\teqi[j] - \meanah{j})^2}{2 \sigma^2  }},\\
       P_j(t) = \left[\frac{1}{4}\ln ( \frac{N_i(t)}{2\pi \varp e})\frac{N_j(t)}{N_i^2(t)}  -\frac{3}{4} \frac{N_j(t)}{N_i^2(t)} + \frac{(\teqi[j] - \meanah{j})^2}{4 \sigma^2  } \frac{N_j^2(t)}{N_i^2(t)}  \right]\\
       \text{and } Q_j(t) =  \left[ \frac{1}{4}\ln( \frac{N_i(t)}{2 \pi \sigma^2 e}) \frac{1}{N_j^2(t)}  + \frac{1}{2N_j(t)} + \frac{1}{N_j(t)}\frac{(\teqi[j] - \meanah{j})^2}{4 \sigma^2  } \right].
    \end{gather*}
Also note that as we pull the arm $k$, we have for any $j\leq i$, 
\begin{equation}\label{eq:comparejk}
    g_j(t)Q_j(t)\leq g_k(t) Q_k(t).
\end{equation}
As a consequence, we can write for any $\delta>0$ and $\tilde{x}_j=\sqrt{\frac{N_j(t)}{2\sigma^2}}(\teqi[j]-\meanah{j})$:
\begin{align*}
    g_j(t)P_j(t) & \leq g_j(t)\left(\ln(t)\frac{N_j(t)}{4N_i^2(t)}\right)+\tilde{x}_j^3 e^{-\tilde{x}_j^2}\frac{N_j(t)}{2N_i^2(t)}\\
    & \leq g_j(t)\left(2\ln(t)+\frac{1}{\delta}\right)\frac{N_j(t)}{2N_i^2(t)} + \frac{N_j(t)}{2N_i^2(t)}\delta ,
\end{align*}
where we used the fact that $\tilde{x}_j^3 e^{-\tilde{x}_j^2}\leq \frac{\tilde{x}_j e^{-\tilde{x}_j^2}}{\delta}+\delta$. 
Moreover, note that for $t(\pmb{\mu})$ large enough, $Q_j(t)\geq \frac{1}{2N_j(t)}$. As a consequence, 
\begin{align*}
    g_j(t)P_j(t) & = g_j(t)\left(2\ln(t)+\frac{1}{\delta}\right)\frac{N_j(t)}{2N_i^2(t)Q_j(t)} Q_j(t) + \frac{N_j(t)}{2N_i^2(t)}\delta \\
    &\leq Q_j(t)g_j(t)\left(2\ln(t)+\frac{1}{\delta}\right) + \frac{K\delta}{2t}\\
    & \leq Q_k(t)g_k(t)\left(2\ln(t)+\frac{1}{\delta}\right) + \frac{K\delta}{2t},
\end{align*}
where the last inequality comes from \cref{eq:comparejk}.
    In particular, \begin{align}
    S_k &\leq K\left(2\ln(t)+\frac{1}{\delta}\right)Q_k(t)g_k(t)+\frac{K^2\delta}{2t}.
    \label{eq:S2bound}
\end{align}
Also, $S_i\geq \frac{1}{2t}-\frac{K^2}{4t^2}$ since $N_i(t)\geq\frac{t}{K}$ and $\ln(1+x)\geq x-\frac{x^2}{2}$ for $x\in[0,1]$. 

\medskip

Now assume that $\mean_k-\meana{k}(t)\geq-\varepsilon\Delta_k$. It then holds 
\begin{align*}
    \teqi[k] - \meana{k}(t) & \geq \meana{i}(t)-\meana{k}(t)\\
    & \geq \Delta_k + \mean_k-\meana{k}(t)-\sqrt{\frac{6\varp\ln t}{\sqrt{t}}}\\
    & \geq (1-\varepsilon)\Delta_k-\sqrt{\frac{6\varp\ln t}{\sqrt{t}}}.
\end{align*}
Again, we can choose $t(\pmb{\mu})$ large enough so that $\teqi[k] - \meana{k}(t)\geq (1-2\varepsilon)\Delta_k$. Moreover, note that the functions $x\mapsto \frac{e^{-x^2}}{x}, x\mapsto xe^{-x^2}, x\mapsto x^3e^{-x^2}$ are all decreasing on an interval of the form $[M,+\infty]$. 
As a consequence, we can choose $n(\pmb{\mu})$ large enough so that $\frac{\sqrt{n(\pmb{\mu})((1-2\varepsilon)\Delta_k)^2}}{\sqrt{2\varp}}\geq M$. If $N_k(t)\geq n(\pmb{\mu})$, we then have from \cref{eq:S2bound}, for a constant $c(K,\Delta_k)$ solely depending on $K$ and $\Delta_k$:
\begin{align*}
    S_k\leq c(K,\Delta_k)e^{-\frac{N_k(t)(1-2\varepsilon)^2\Delta_k^2}{2\varp}}\left[\ln t +\frac{1}{\delta}\right]+\frac{K^2\delta}{2t}.
\end{align*}
The inequality $S_k\geq S_i$ then implies, thanks to the above bounds:
\begin{gather*}
    c(K,\Delta_k)e^{-\frac{N_k(t)(1-2\varepsilon)^2\Delta_k^2}{2\varp}}\left[\ln t +\frac{1}{\delta}\right]\geq \frac{1-K^2\delta}{2t}-\frac{K^2}{4t^2}.
\end{gather*}
In particular, for $\delta=\frac{1}{2K^2}$,
\begin{gather*}
    N_k(t) \leq \frac{2\varp}{(1-2\varepsilon)^2\Delta_k^2} \left(\ln t +  \ln\ln t+\bigO{1}\right),
\end{gather*}
where the $\mathcal{O}$ hides constants depending in $K$ and $\Delta_k$. This concludes the proof of \cref{lemma:inclusion} as we just shown that if $\cE_k(t)$ holds, at least one of the two following events holds when $N_k(t)\geq n(\pmb{\mu})$:
\begin{itemize}
    \item $\mean_k-\meana{k}(t)\leq-\varepsilon\Delta_k$
    \item $ N_k(t) \leq \frac{2\varp}{(1-2\varepsilon)^2\Delta_k^2} \left(\ln t +  \ln\ln t+\bigO{1}\right)$.
\end{itemize}
\end{proof}

\section{Generalization of the information maximization approximation}\label{expofamily}

In this section, we will generalize the approach derived in \cref{analyticalapprox} to bandit settings with a reward distribution belonging to the exponential family. We will retrace all the previous steps made in \cref{analyticalapprox}, insisting on the differences with the Gaussian reward case. We will also discuss bandit settings with non-uniform priors and with more than two arms.

\subsection{Asymptotic expression for exponential family rewards}\label{app:exponential}

We derive an asymptotic expression for the one-dimensional canonical exponential family from which we will derive an analytical approximation of the entropy. We thus focus on a reward distribution density $f$ with respect to some reference measure $\nu$ belonging to some one-dimensional canonical exponential family, i.e.,
\begin{equation}
    f(x| \theta) = A(x) \exp\big(T(x)\theta - F(\theta)\big),
\end{equation}
where $F$ is twice differentiable and strictly convex. Additionally, let us recall that the \kulleib divergence verifies: \citep{kordaThompsonSampling1dimensional2013}:
\begin{equation}\label{kullexp}
    \kullfct(\theta, \theta')= F(\theta') - F(\theta) - F'(\theta)(\theta' - \theta),
\end{equation}
where $\kullfct(\theta, \theta')$ is the Kullback-Leibler divergence between the reward distribution parameterized by $\theta$ and the one parameterized by $\theta'$.

Given a prior $\pi (\theta)$ and the reward realizations $( x_1 , \ldots, x_n )$, the associated posterior distribution on $\theta$, denoted $p$,  reads:
\begin{equation}\label{Siposteriorexp}
    p( \theta | x_1,..,x_n ) = \frac{1}{C}\pi (\theta) \exp \left( \theta \sum_{k=1}^n T(x_k) - n F(\theta) \right),
\end{equation}
where $C = \int \pi (\theta) \exp \left( \theta \sum T(x_k) - n F(\theta) \right)\df\theta $ is a normalization constant. Next, we derive the maximum a posteriori for the parameter $\theta$, denoted $\meanlog$, which verifies: 
\begin{equation}\label{Sithetaldef}
 \sum_{i=1}^n T(x_i) = nF'(\meanlog)  - \frac{\pi' (\meanlog)}{\pi(\meanlog)} .
\end{equation}
At this stage, we assume that there exits such $\meanlog$ verifying \cref{Sithetaldef}. In practice, for a reward distribution that does not meet this criteria, one can replace $\hat{\theta}_l$ by a series $\hat{\theta}_{n,l}$ which, for sufficiently large values of $n$, asymptotically conforms to the aforementioned definition. For example, a Bernoulli arm that consistently fails under a uniform prior, will result in an undefined $\meanlog$. To address this, one may redefine  $\meanlog$ such that $(1 + \sum{i=1}^n T(x_i)) = (n+2)F'(\meanlog)$, effectively replacing the empirical mean in \cref{Sithetaldef} with the posterior mean.  

Replacing the sum in \cref{Siposteriorexp} leads to 
\begin{equation}\label{Siposteriorexp2}
\begin{split}
    p( \theta |  x_1,..,x_n ) &= \frac{1}{C}\pi (\theta) \exp\left( \theta n  F'(\meanlog) - \theta \frac{\pi' (\meanlog)}{\pi(\meanlog)} - n F(\theta) \right)  \\
    &= \frac{e^{n \meanlog F'(\meanlog) - n F(\meanlog) }}{C} \pi (\theta) e^{- \theta \frac{\pi' (\meanlog)}{\pi(\meanlog)} } e^{ -n \kullfct(\meanlog, \theta) } \\
    &= \frac{1}{C_2} \pi (\theta) e^{- \theta \frac{\pi' (\meanlog)}{\pi(\meanlog)} } e^{ -n \kullfct(\meanlog, \theta) },\\
\end{split}
\end{equation}
where $C_2$ also acts as a normalization constant of \cref{Siposteriorexp2}. For $n \gg 1$, the distribution concentrates in the vicinity of $\meanlog$ from which we will derive the asymptotic scaling of $C_2$. We then integrate \cref{Siposteriorexp2} after a change of variable $\theta(u) = \meanlog + \frac{u}{\sqrt{n}}$, 
\begin{equation}\label{Siposteriorexp3}
\begin{split}
    1= \int_{\thetaDomain} p( \theta |x_1,..,x_n ) \df\theta &= \int_{-(\theta_b -\meanlog) \sqrt{n}}^{(\theta_b-\meanlog) \sqrt{n}} \frac{1}{C_2\sqrt{n}} \pi (\meanlog + \frac{u}{\sqrt{n}}) e^{-(\meanlog + \frac{u}{\sqrt{n}}) \frac{\pi' (\meanlog )}{\pi(\meanlog )} } e^{ -n \kullfct(\meanlog, \meanlog + \frac{u}{\sqrt{n}}) } \df u\\ &\phantom{=}+  \int_{\binf}^{-\theta_b} p( \theta |x_1,..,x_n )\df\theta + \int_{\theta_b}^{\bsup} p( \theta |x_1,..,x_n ) .\df\theta
\end{split}
\end{equation}

Taking  $(\theta_b -\meanlog) \sim n^{-b}$ with $b<1/2$, we get rid of the tail components in the asymptotic limit. Secondly, by noticing that  $   \frac{F''(\meanlog)}{2}= \underset{\theta \rightarrow \meanlog}{\lim} K(\meanlog,\theta) / |\theta- \meanlog|^2$ from \cref{kullexp}, we make an expansion to the lowest order of the \kulleib divergence, which gives:
\begin{equation}\label{Siposteriorexp4}
\begin{split}
    1 &= \underset{\theta \rightarrow \meanlog}{\lim} \int_{-(\theta_b -\meanlog) \sqrt{n}}^{(\theta_b-\meanlog) \sqrt{n}} \frac{1}{C_2\sqrt{n}} \pi (\meanlog ) e^{-\meanlog  \frac{\pi' (\meanlog )}{\pi(\meanlog)} } e^{ -\frac{F''(\meanlog)u^2}{2}  \meanlog + \bigO{\frac{1}{\sqrt{n}}} } \df u.\\ 
\end{split}
\end{equation}

Thus, we obtain:
\begin{equation}\label{Siposteriorexp5}
    C_2 \sim \frac{\sqrt{2 \pi }}{\sqrt{ n F''(\meanlog)} }  \pi (\meanlog ) e^{-\meanlog  \frac{\pi' (\meanlog )}{\pi(\meanlog)} }.
\end{equation}

Of note, the gaussian limit also gives that  $\Vi \sim F''(\meanlog)^{-1} \ndraw{i}^{-1}$.


Thus, we assume to develop an approximation scheme for a posterior distribution $p_i$ asymptotically verifying:
\begin{equation}\label{generalposterior}
\begin{split}
 p_i(\theta) & \underset{ \ndraw{i} \rightarrow \infty}{\sim} \sqrt{ \frac{  1}{2 \pi \Vi}}  H(\theta, \meanlog)  e^{ - \ndraw{i} \kull( \meanlog,\theta)},
\end{split}
\end{equation}
where $H$ is a function accounting for the prior distribution. For the following, we take a uniform prior on $\thetaDomain$, which leads to $ H(\theta, \meanlog) = 1$. 

In the following, we will denote $\meanmax$ and $\meanmin$ as the maximum a posteriori estimates associated to their respective arms (instead of the empirical means). 

\subsection{The partitioning approximation}

Since all the steps leading to the partitioning approximation are independent of the type of reward distribution, $\Stail$ and $\Sc$ have the same general form as given in \cref{firstanalyticalapprox}. Here, we consider all the distributions under $\theta$ parameter for which we replace $\meanmax, \meanmin$ and $\teq$ by there equivalents $\thetamin, \thetamax$ and $\thetaeq$. 

\subsection{ Asymptotic intersection point}

By use of \cref{generalposterior}, the equation verified by the intersection point $\thetap$ asymptotically reads:
\begin{equation}\label{geSithetaeq1}
\frac{e^{ - \Nmax \kull(\thetamax,\thetap)}}{\sqrt{2 \pi \Vmax }} \int_{\binf}^{\thetap}  \frac{ e^{- \Nmin \kull(\thetamin,\theta')}}{\sqrt{2 \pi \Vmin}}  \df\theta' =  \frac{ e^{- \Nmin \kull(\thetamin,\thetap)}}{\sqrt{2 \pi \Vmin}} \int_{\binf}^{\thetap} \frac{e^{ - \Nmax\kull(\thetamax,\theta')}}{\sqrt{2 \pi \Vmax }} \df\theta'.
\end{equation}

Taking the logarithm of \cref{geSithetaeq1} leads to
\begin{equation}\label{Sithetaeqgeneral}
\begin{split}
& \Nmin \kull(\thetamin,\thetap)- \Nmax \kull(\thetamax,\thetap)  +  \frac{1}{2}\ln \frac{\Vmin}{\Vmax} + \ln \frac{\int_{\binf}^{\thetap} \sqrt{\Vmax} e^{- \Nmin \kull(\thetamin,\theta')} \df\theta'}{ \int_{\binf}^{\thetap} \sqrt{\Vmin}e^{ - \Nmax\kull(\thetamax,\theta')} \df \theta'} = 0.
\end{split}
\end{equation}

Employing the same arguments as the ones exposed in \cref{refteqapprox}, we approximate $\thetap$ by neglecting the last term. Furthermore, in the considered asymptotic scaling regime ($\Nmax \gg \Nmin$), $\thetap$ will be in the vicinity of $\thetamax$ where a Gaussian expansion of the \kulleib divergence is relevant (see \cref{Siposteriorexp3}). Thus, we approximate $\kull(\thetamin,\thetaeq)$ by $\kull(\thetamin,\thetamax)$ and expand $\kull(\thetamax,\thetap)$ to lowest order in $\thetaeq$ (with $\Vi \sim F''(\meanlog)^{-1} \ndraw{i}^{-1}$), leading to:
\begin{equation}\label{geSithetaeq2}
\begin{split}
&\thetaeq=  \thetamax +\sqrt{ 2 \Vmax \left[ \Nmin \kull(\thetamin,\thetamax) + \frac{1}{2}\ln \frac{\Vmin}{\Vmax} \right]}.
\end{split}
\end{equation}

\subsection{Generalization of the main mode's contribution}

We start by recalling the expression for the body component  of the entropy:
\begin{equation}\label{geSiSb1}
\Sc = - \int_{\thetaDomain} \parmax(\val) \cumin(\val)  \ln\parmax(\val) \df \val. 
\end{equation}

Without any additional information on the expression for $\kull$, \cref{geSiSb1} cannot be computed in a closed form. Thus, we will rely on the asymptotic scaling  $\Nmax \gg \Nmin \gg 1$ to provide a tractable expression. First, we neglect variations of $\cumin(\val)$ in \cref{geSiSb1} integral by evaluating it at $\teq$. Then, by noticing that the resulting integral is the entropy of the better empirical arm's mean, we approximate it by its leading order, proportional to $\ln(2 \pi \Vmax)/2$:
\begin{equation}\label{geSiSb2}
\begin{split}
\Sc &\approx \frac{1}{2}\ln ( 2\pi  \Vmax) \left[ 1 - \int_{\thetap}^{\bsup}  \frac{ e^{- \Nmin \kull(\thetamin,\theta')}}{\sqrt{2 \pi \Vmin}}  \df\theta'  \right]. \\
\end{split}
\end{equation}

We finally consider the last integral in \cref{geSiSb2}. By noticing that it is concentrated in the vicinity of $\teq$ for $\Nmin \gg 1$, we Taylor expand $\kull(\thetamin,\theta')$ at $\teq$ to obtain: 
\begin{equation}\label{geSiSb3}
\begin{split}
\int_{\thetap}^{\bsup}  \frac{ e^{- \Nmin \kullfct(\thetamin,\theta')}}{\sqrt{2 \pi \Vmin}}  \df\theta &\approx  \frac{ e^{- \Nmin \kullfct(\thetamin,\teq)}}{\sqrt{2 \pi \Vmin}}  \int_{\thetap}^{\bsup}  e^{- \Nmin (\theta'-\teq) \partial_2\kullfct(\thetamin,\teq)} \df\theta'\\
&\approx  \frac{ e^{- \Nmin \kullfct(\thetamin,\teq)}}{\sqrt{2 \pi \Vmin} \Nmin \partial_2 \kullfct(\thetamin,\teq)}. \\
\end{split}
\end{equation}

Inserting \cref{geSiSb3} into \cref{geSiSb2} leads to the  body component:
\begin{equation}\label{geSiSbf}
\begin{split}
\Scapp &= \frac{1}{2}\ln ( 2\pi \Vmax)  \left[ 1 - \frac{ e^{- \Nmin \kullfct(\thetamin,\thetaeq)}}{\sqrt{2 \pi \Vmin} \Nmin \partial_2\kullfct(\thetamin,\thetaeq)} \right]. \\
\end{split}
\end{equation}

\subsection{ Generalized expression for the entropy tail}

We start by recalling the expression for  the tail component:
\begin{equation}\label{geSiStail1}
\Stail = - \int_{\thetaeq}^{\bsup} \parmin(\val) \ln \parmin(\val) \df\val.
\end{equation}

As for \cref{geSiSb3}, we Taylor expand $\kull(\thetamin,\theta')$ at $\thetaeq$ in the exponential term to obtain: 
\begin{equation}\label{geSiSatil1}
\begin{split}
\Stailapp &= -\ln \parmin(\thetaeq) \frac{ e^{- \Nmin \kullfct(\thetamin,\thetaeq)}}{\sqrt{2 \pi \Vmin} \Nmin \partial_2\kullfct(\thetamin,\thetaeq)}.
\end{split}
\end{equation}

Keeping the leading order of $-\ln \parmin(\thetaeq) \sim \Nmin \kullfct(\thetamin,\thetaeq)$ leads to the expected tail expression used in the main text.

\subsection{ Generalized form of the entropy approximation}\label{genetaildiscussion}

To summarize, by combining \cref{geSiSbf,geSiSatil1} we obtain an asymptotic expression for exponential family bandits with a uniform prior:
\begin{equation}\label{geSiSappgeneral}
\begin{split}
\Sapp &= \frac{1}{2}\ln ( 2\pi  \Vmax)  \left[ 1 - \frac{ e^{- \Nmin \kullfct(\thetamin,\thetaeq)}}{ \Nmin \partial_2\kullfct(\thetamin,\thetaeq) \sqrt{2 \pi \Vmin}} \right] +   \frac{\kullfct(\thetamin,\thetaeq)  e^{- \Nmin \kullfct(\thetamin,\thetaeq)}}{ \partial_2\kullfct(\thetamin,\thetaeq) \sqrt{2 \pi \Vmin}  }.
\end{split}
\end{equation}

Finally, depending implementation, we propose for convenience to replace in $\Sapp$ the maximum a posteriori estimates of each arm by either their empirical mean, their mean posterior values or the maximum of the log-likelihood. This does not alter the algorithm's efficiency in practice, while it may simplify the implementation procedure for specific reward distributions.  

Note that all these steps can be adapted to non-uniform priors (in particular by multiplying the tail by the prior effects evaluated in $\thetaeq$). Finally, let us underline that our approximation scheme holds for any posterior distributions verifying  \cref{generalposterior}, a property we believe to be shared for more general reward distributions.

\subsection{Derivation of the increment for the closed-form expression of entropy}\label{Siincremententrapemant}

First, we stress there is no unique guideline to compute the expected increment of \cref{geSiSappgeneral}, and multiple solutions emerge depending on the type of the reward distribution. In particular, if the reward distribution is continuous, one could integrate the increment as it has been done for Gaussian rewards above. But, if the integration cannot be solved analytically, one could approximate the increments by taking discrete reward values of the order of $ \pm \sigma$. Similarly, if the reward takes discrete values, the increments are already discrete, but asymptotic simplifications or taking the continuous limit can also be considered. 

Finally, if the increment evaluation is discrete or approximated, one could encounter rare cases where the algorithm gets trapped. It could occur when the algorithm observes a worse suboptimal arm close to the best empirical arm when it has already extensively been drawn. Because the entropy could increase drastically if an arm inversion occurs, the gradient signs may occasionally switch, leading to the failure of the minimization procedure. To prevent such cases, we change the decision procedure by maximizing the entropy variation rather than its direct minimization. An example is given for the implementation of Bernoulli rewards in the next section.

Lastly, depending on the reward distribution it may be straightforward to express the increments along the usual empirical or posterior mean as opposed to the family parameter $\hat{\theta}$.  Often, this can be achieved through a basic variable transformation. The Bernoulli distribution example provided below serves as an illustration of this approach.

\section{Numerical experiments}\label{app:expe}

Here, we provide all the information regarding numerical experiments. This includes details on the numerical settings, implementation details for \algoname in the investigated settings, an overview of investigated classical bandit algorithms, and additional experiments focusing on close-arm means.

\subsection{Numerical settings}

In \cref{fig:1pres}, the posterior distributions are drawn with $\meanmax \approx0.65$,  $\ndrawt{1}=374$, $\meanmin \approx 0.29$, $\Nmin=26$,  where $\mean_{i}$ and $\ndrawt{i}$ are, respectively, the empirical mean and number of draws of arm $i$ and have been obtained with the AIM algorithm.

For the Gaussian two-armed cases in \cref{fig:1} the arm means are chosen from a uniform grid in $(0,1)\times(0,1)$ using a Sobol sequence (we have avoided the  values 0 and 1 but it has no impact on the obtained results). The regret is averaged over $8192$ games and observed during $10^8$ steps to attest the logarithmic scaling. For Bernoulli rewards with two-armed \cref{fig:3}, the regret is averaged over $16384$ games and observed during $10^8$ steps. It is worth noting that for Gaussian rewards, the prior information of arm means being only between $0$ and $1$ is not given to \algoname nor to the Thompson sampling algorithm to allow a direct comparison.

For the fifty-armed case in \cref{fig:1,fig:3}, the arm means are drawn from a uniform prior, and the regret is averaged over $2000$ games and observed during $10^6$ steps in \cref{fig:1} and $10^7$ in \cref{fig:3}. 

For close arm means in \cref{fig:closemean,fig:closemeanber}, the mean values are fixed with $\mu_1 = 0.79$ and $\mu_2= 0.8$, but this prior information is not given to the investigated algorithms. The regret is averaged over $10^5$ games and observed during $10^6$ steps. 

For the two-armed cases in \cref{fig:1moredata,fig:2moredata}, the arm means are chosen from a uniform grid in $(0,1)\times(0,1)$ using a Sobol sequence. The regret is averaged over more than $10^5$ games and observed during $10^6$ steps to enhance measurement accuracy.

Finally, in the fifty-armed case in \cref{fig:3moredata}, the arm means are drawn from a uniform prior, and the regret is averaged over $4.10^4$ games and observed during $5.10^4$ steps to enhance measurement accuracy.

Of note, for all the experiments, seed values are not shared throughout the algorithms. To obtain a sufficient number of runs, the code was parallelized on a cluster (asynchronously), with each run operating independently while ensuring that seed values are not common between runs.
Because it relies on an analytical expression, AIM shows an execution time of the same order of Thompson sampling (measured three times slower for two-armed Bernoulli rewards).

For completeness, an implementation of AIM for both Bernoulli and Gaussian rewards and more than two arms are given in the supplementary material (\texttt{AIM Bernoulli bandits and AIM Gaussian} folders). 

\subsection{AIM implementation details}

Here, we recap below \algoname setups for the different settings evoked in the main text. 

%
%

\subsubsection{Approximate information maximization for the two arm Gaussian rewards}

Specifically for the two-armed case, one can simplifies the expressions given in the main text. $\teq$ defined the value of $\theta$ where both arms have the same probability of being the maximal one and reads 
\begin{equation}\label{thetaeqtwoarmlast}
\begin{split}
&\teq = \meanmax + \frac{\Nmin(\meanmax -\meanmin)}{\Nmax-\Nmin}+ \sqrt{ \frac{\Nmax \Nmin( \meanmax-\meanmin)^2}{(\Nmax - \Nmin)^2}  + \frac{\sigma^2 }{\Nmax- \Nmin}\ln \left( \frac{\Nmax}{\Nmin} \right)  } .
\end{split}
\end{equation}
Hence, following identical approximations of the ones derived in \cref{firstanalyticalapprox}for Gaussian reward distributions, the tail component given by \cref{Stailexp} simplifies into:

\begin{equation}\label{Stailgauss}
\begin{split}
\Stail &=  \frac{1}{4} \ln( \frac{2\pi \varp e}{\Nmin}) \erfc \left(\frac{\sqrt{\Nmin} (\teq-\meanmin)}{\sqrt{2 \varp}} \right) + \frac{\sqrt{\Nmin}(\teq-\meanmin)}{2 \sqrt{2\pi \varp}} e^{ -\frac{-\Nmin(\teq-\meanmin)^2 }{2 \varp}}.\\
\end{split}
\end{equation}
Similarly for $\Sc$, we obtain:
\begin{equation}\label{Sbodygauss}
\begin{split}
\Sc &= \frac{1}{2}\ln \left( \frac{2\pi \varp e}{\Nmax} \right) \times \left[ 1 - \frac{1}{2} \erfc \left(  \frac{ \sqrt{\Nmin \Nmax} (\meanmax - \meanmin)}{\sqrt{2\varp (\Nmin + \Nmax)}} \right) \right] \\ &\hspace{5cm}- \frac{ \sqrt{\Nmax} \Nmin^{3/2}  (\meanmax - \meanmin) }{2 \sigma  \sqrt{2 \pi }  (\Nmax + \Nmin)^{3/2}} e^{-\frac{ \Nmin \Nmax (\meanmax - \meanmin)^2}{2\varp (\Nmin + \Nmax)} }.\\
\end{split}
\end{equation}
Finally, it allows us to derive the approximation of the gradient difference of the entropy for the two-armed case:

\begin{equation}\label{Sgradientsimptwoarmlast}
\begin{split}
 &\Delta =  \frac{1}{2}\ln (\frac{\Nmin}{\Nmin +1}) +\frac{1}{4 \Nmax}\erfc \left(  \frac{ \sqrt{\Nmin } (\teq - \meanmin)}{\sqrt{2\varp}} \right) + \frac{ \sqrt{\Nmin } (\teq - \meanmin)}{\sqrt{ 2\pi \varp}} \times \\ &e^{-\frac{\Nmin  (\teq - \meanmin)^2}{2\varp  }} \bigg[  \frac{2\Nmax^2 - 3\Nmin^2}{4 \Nmin \Nmax^2}  + \frac{1}{4}\ln \left( \frac{\Nmax}{2\pi \varp e} \right) \frac{\Nmin^3 + \Nmax^2}{\Nmin^2\Nmax^2}   + \frac{(\Nmin^3 + \Nmax^2)(\teq - \meanmin)^2}{4 \sigma^2 \Nmin \Nmax^2  }   \bigg],  \\ 
\end{split}
\end{equation}
and the associated pseudo-code used for \cref{fig:1,fig:1moredata,fig:closemean} is presented in \cref{alg:bandit_algorithm_twoarm } below.
\begin{algorithm}
\caption{\algoname Algorithm for $2$ Gaussian arm}
\label{alg:bandit_algorithm_twoarm }
Draw each arm once; observe reward $\Rewardt(t)$ and update statistics
$\meana{t} \leftarrow \Rewardt(t), \ndraw{t} \leftarrow 1 \, \, \forall t \in \{1,2\}$

\For{$t=3$ \KwTo $T$}{
    \tcc{Arm selection}
    $\armmax \gets \argmax_{k=1,2} \meana{k}$, $\armmin \gets \argmin_{k=1,2} \meana{k}$\;
    \If{$\Nmax \leq \Nmin$}{
    $\Armt\gets\armmax$}
    \Else{
    Evaluate $\Delta$ following \cref{Sgradientsimp} \;

    \If{$\Delta < 0$}{
    $\Armt\gets\armmax$}
    \Else{
    $\Armt\gets\armmin$}
    }
    Pull $\Armt$ and observe $\Rewardt(\Armt)$\\\tcc{Update statistics}
    $\meana{\Armt} \leftarrow \frac{\meana{\Armt}\ndraw{\Armt} +\Rewardt(\Armt)}{ \ndraw{\Armt} + 1} $,
    $\ndraw{\Armt} \leftarrow \ndraw{\Armt} + 1 $
}
\end{algorithm}

\subsubsection{Approximate information maximization for Bernoulli rewards}

We denote by $\meanab{i}$ the posterior mean, given by:
\begin{equation}\label{Simeanposterior}
\mathbb{E}\left[ X_{\mathcal{B}(\rew{i}+1, \ndraw{i} -  \rew{i} +1)}\right] = \frac{ \rew{i}+1}{\ndraw{i}+2} = \meanab{i},
\end{equation}
where $\rew{i}$ is the cumulative reward at time $t$ , $\ndraw{i}$ the number of draws and $X_{\mathcal{B}(a, b)}$ follows a Beta distribution with parameters $(a,b)$. 
The variance verifies:
\begin{equation}\label{Sivarposterior}
\begin{split}
\mathrm{Var}\left[X_{\mathcal{B}(\rew{i}+1, \ndraw{i} -  \rew{i} +1)} \right] &= \frac{ \rew{i}+1}{\ndraw{i}+2} \frac{ \ndraw{i} - \rew{i}+1}{\ndraw{i}+2}  \frac{1}{\ndraw{i}+3}\\ &= \frac{ \meanab{i}(1- \meanab{i})}{ \ndrawb{i}},\\
\end{split}
\end{equation}
where $\ndrawb{i} = \ndraw{i} + 3$.

For Bernoulli rewards, we approximate the gradient as follows:
\begin{equation}\label{Sgradientbernoulli}
\begin{split}
|\Delta_i \Sapp| = \bigg| &\frac{\meanab{i}(\ndrawb{i} - 1)-1}{\ndrawb{i} - 3} \Sapp(  \frac{\meanab{i}\ndrawb{i} + 1-  \meanab{i} }{\ndrawb{i}} , \ndrawb{i} + 1, \meanab{j}, \ndrawb{j})\\ &+ \frac{\ndrawb{i} - 2 -\meanab{i}(\ndrawb{i}- 1) }{\ndrawb{i}-3} \Sapp( \frac{\meanab{i}(\ndrawb{i}-1) }{ \ndrawb{i}}, \ndrawb{i} + 1, \meanab{j}, \ndrawb{j})  - \Sapp(\meanab{i},\ndrawb{i},\meanab{j}, \ndrawb{j})  \bigg|,\\
\end{split}
\end{equation}
with $\Sapp$ given by \cref{geSiSappgeneral} expressed along $\mu$ with $\mu = \frac{e^{\theta}}{1  + e^{\theta}}$. For Bernoulli rewards the equation reads:
\begin{equation}\label{Sappbernoulli}
\begin{split}
\Sapp(\meanab{\maxa}, \ndrawb{\maxa}, \meanab{\mina}, \ndrawb{\mina}) &= \left( 1 - \frac{e^{- \ndrawb{\mina}\kullfct( \meanab{\mina}, \teq)}}{ \sqrt{\Nmin}\partial_2\kullfct( \meanab{\mina}, \teq) \sqrt{2 \pi \meanab{\mina}(1-\meanab{\mina})}} \right) \frac{1}{2} \ln \left( \frac{2 \pi \meanab{\maxa}(1-\meanab{\maxa})}{\ndrawb{\maxa}} \right) \\ & \hspace{4cm }+  \frac{\sqrt{\ndrawb{\mina}}\kullfct( \meanab{\mina}, \teq) e^{- \ndrawb{\mina}\kullfct( \meanab{\mina}, \teq)}}{ \partial_2\kullfct( \meanab{\mina}, \teq) \sqrt{2 \pi \meanab{\mina}(1-\meanab{\mina})}},\\
\end{split}
\end{equation}
with $\kullfct( \theta, \theta') = \theta \ln(\theta/\theta') + (1-\theta) \ln([1-\theta]/ [1-\theta'])$ and  $\sigma_i^2 = \frac{\meanab{i}(1-\meanab{i})}{\ndrawb{i}}$.

Briefly, the expected gradient is evaluated along arm $i$ with a returned reward equal to $1$ with probability $\frac{\meanab{i}(\ndrawb{i} - 1)-1}{\ndrawb{i} - 3}$ (which is the empirical mean) or equal to $0$ with probability $1- \frac{\meanab{i}(\ndrawb{i} - 1)-1}{\ndrawb{i} - 3}$. 

Of note, by adding absolute values, we seek to maximize the entropy variation rather than its direct minimization to avoid falling into an entrapment scenario (see \cref{Siincremententrapemant} for further discussion). 

\begin{algorithm}[htbp]
\caption{\algoname Algorithm for $2$ Bernoulli arm}
\label{alg:bandit_algorithmber}
Draw each arm once; observe reward $\Rewardt(t)$ and update statistics
$\meanab{t} \leftarrow \frac{\Rewardt(t) + 1}{3}, \ndrawb{t} \leftarrow 4\, \, \forall t \in \{1,2\}$

\For{$t=3$ \KwTo $T$}{
    \tcc{Arm selection}
    $\armmax \gets \argmax_{k=1,2} \meanab{k}$, $\armmin \gets \argmin_{k=1,2} \meanab{k}$\;
    \If{$\Nmax \leq \Nmin$}{
    $\Armt\gets\armmax$}
    \Else{
    Evaluate $\Delta = |\Delta_{\armmax} \Sapp| - |\Delta_{\armmin} \Sapp| $ following \cref{Sgradientbernoulli} \;

    \If{$\Delta > 0$}{
    $\Armt\gets\armmax$}
    \Else{
    $\Armt\gets\armmin$}
    }
    Pull $\Armt$ and observe $\Rewardt(\Armt)$\\
    \tcc{Update statistics}
    $\meanab{\Armt} \leftarrow \frac{\meanab{\Armt}(\ndrawb{\Armt}- 1) +\Rewardt }{ \ndrawb{\Armt}} $,
    $\ndrawb{\Armt} \leftarrow \ndrawb{\Armt} + 1 $
}
\end{algorithm}

We draw some additional observations on the practical implementation of the code. First, in the gradient evaluation of $\Delta_{i}$ following \cref{Sgradientbernoulli} we may find a $\teq$ value to be undefined (because $\Nmin > \Nmax$ or $\teqi > 1$), which is unusable for Bernoulli rewards. In this case, $\teqi$ is taken to be equal to $1$, resulting in $\Smax = \frac{1}{2} \ln \left( \frac{2 \pi \meanab{\maxa}(1-\meanab{\maxa})}{\ndrawb{\maxa}} \right)$. 

Second, at large times, noticing that the better empirical arm is drawn extensively, one can increment the algorithm by multiple steps at a time to speed up \algoname. 
Indeed, let us assume that the better empirical arm is drawn $T$ times successively while always returning a null reward, which is the worst scenario for the returned reward of the better empirical arm. Then, if the increment evaluation at $t+T$ of \cref{alg:bandit_algorithmber} still returns $\armmax$, then it ensures that all increment evaluations between $[t,t+T]$ of \cref{alg:bandit_algorithmber} will always return $\armmax$ independently of its returned rewards. Then, using a dichotomy search on the variable $T$, we can diminish the number of increment evaluations of \algoname at large times, thus improving \algoname['s]{} performance.

\subsection{Information maximization approximation for Bernoulli rewards with more than two arms}\label{multiimplemtationdetail}

We start by reminding the obtained entropy approximation for more than two arms:

\begin{equation}\label{Sappexpmulit}
\Sapp= - \int_{\thetaDomain}  \big( 1 - \sum_{i \neq \armmax}^{K} [1- \cumi(\val)] \big) \parmax(\val)  \ln\parmax(\val)  \df\val - \sum_{i \neq \armmax}^{K} \int_{\teqi}^{\bsup} \parmi(\val) \ln \parmi(\val)  \df\val.
\end{equation}

We first consider the increment along a worse empirical arm, which simplifies :

\begin{equation}\label{Sappexpmulit2}
|\Delta_{i} \Sapp|=  \Delta_i \bigg[  - \int_{\thetaDomain} \cumi(\val) \parmax(\val)  \ln\parmax(\val) \df\val - \int_{\teqi}^{\bsup} \parmi(\val) \ln \parmi(\val) \df\val \bigg] ,
\end{equation}

which is exactly the increment evaluated in the two-armed case given in \cref{Sappbernoulli}.

Finally, we consider the increment along the better empirical arm. For simplicity, we neglect $\teqi$ variations for the increments evaluation. By use of \cref{Sappbernoulli} we obtain

\begin{equation}\label{Sappexpmulit1}
| \Delta_{\armmax} \Smax |=  \bigg|  1 -  \sum_{i \neq \armmax}^{K} \frac{e^{- \ndrawb{\mina}\kullfct( \meanab{i}, \teq)}}{ \sqrt{\Nmin}\partial_2\kullfct( \meanab{i}, \teq) \sqrt{2 \pi \meanab{i}(1-\meanab{i})}}  \bigg|  \bigg| \Delta_{\armmax} H(\meanab{\maxa}, \ndrawb{\maxa})  \bigg|,
\end{equation}

where 

\begin{equation}\label{Sgradientbodybern}
\begin{split}
\bigg| \Delta_{\armmax} H(\meanab{i}, \ndrawb{i}) \bigg| = \bigg| &\frac{\meanab{i}(\ndrawb{i} - 1)-1}{\ndrawb{i} - 3} H(  \frac{\meanab{i}\ndrawb{i} + 1-  \meanab{i} }{\ndrawb{i}} , \ndrawb{i} + 1, \meanab{j}, \ndrawb{j})\\ &+ \frac{\ndrawb{i} - 2 -\meanab{i}(\ndrawb{i}- 1) }{\ndrawb{i}-3} H( \frac{\meanab{i}(\ndrawb{i}-1) }{ \ndrawb{i}}, \ndrawb{i} + 1, \meanab{j}, \ndrawb{j})  - H(\meanab{i},\ndrawb{i},\meanab{j}, \ndrawb{j})  \bigg|,\\
\end{split}
\end{equation}

with $H(\meanab{\maxa}, \ndrawb{\maxa}) = \frac{1}{2} \ln \left( \frac{2 \pi \meanab{\maxa}(1-\meanab{\maxa})}{\ndrawb{\maxa}} \right)$.

\begin{algorithm}[htbp]
\caption{\algoname Algorithm for $K > 2$ Bernoulli arm}
\label{alg:bandit_algorithmmulti}
Draw each arm once; observe reward $\Rewardt(t)$ and update statistics
$\meanab{t} \leftarrow \frac{\Rewardt(t) + 1}{3}, \ndrawb{t} \leftarrow 4\, \, \forall t \in \{1,..K\}$

\For{$t=K+1$ \KwTo $T$}{
    \tcc{Arm selection}
    $\armmax \gets \argmax_{k=\{1,..,K\}} \meanab{k}$;
    Evaluate $\Delta_{\armmax} \Sapp$ following \cref{Sappexpmulit1} \;
    Evaluate $ \mina = argmax (\Delta_{i} |\Sapp|, i \neq \armmax )$ with $\Delta_{i} |\Sapp|$ following \cref{Sgradientbernoulli} \;

    \If{$\Delta_{\armmax} |\Sapp| > \Delta_{\mina}|\Sapp|$}{
    $\Armt\gets\armmax$}
    \Else{
    $\Armt\gets\armmin$}
    
    Pull $\Armt$ and observe $\Rewardt(\Armt)$\\
    \tcc{Update statistics}
    $\meanab{\Armt} \leftarrow \frac{\meanab{\Armt}(\ndrawb{\Armt}-1) +\Rewardt }{ \ndrawb{\Armt}} $,
    $\ndrawb{\Armt} \leftarrow \ndrawb{\Armt} + 1 $
}
\end{algorithm}

Of note in the gradient evaluation of $\Delta_{i}$ following \cref{Sgradientbernoulli}, if ones finds a $\teqi$ value undefined (because $\Nmin > \Nmax$ or $\teqi > 1$ which is unusable for Bernoulli reward), then, $\teqi$ is taken to be equal to $1$ resulting in $\Smax = \frac{1}{2} \ln \left( \frac{2 \pi \meanab{\maxa}(1-\meanab{\maxa})}{\ndrawb{\maxa}} \right)$. Finally, if $\armmax \gets \argmax_{k=\{1,..,K\}} \meanab{k}$ has multiple solutions, we suggest choosing the one displaying the lowest number of draws.

\subsection{Overview of baseline bandit algorithms}\label{app:otheralgos}

Here, we briefly review several baseline algorithms and their chosen parameters to provide a benchmark of our information maximization method.

\subsubsection{UCB-Tuned}

This algorithm falls under the category of upper confidence bound (UCB) algorithms, which select the arm maximizing a proxy function typically defined as $F_i = \meana{i} + B_i$.
For UCB-tuned, $B_i$ is given by:
\begin{equation}\label{UCBTeq}
\begin{split}
R_i &= c(\marmi{1},\marmi{2}) \sqrt{ \frac{\ln(t)}{\ndrawt{i}} \mathrm{min} \left( \frac{1}{4}, s_i(t) \right) },\: \: s_i(t) = \hat{\sigma_i}^2 + \sqrt{ \frac{2\ln(t)}{\ndrawt{i}}},
\end{split}
\end{equation}

where $\hat{\sigma_i}^2$ is the reward variance and $c$ a hyperparameter. For Gaussian rewards, by testing various $c$ values for uniform priors in \cref{UCBTeq}, we end up with $c=2.1$ and  $\hat{\sigma_i}^2 = \frac{\sigma^2}{\ndrawt{i}}$.

\subsubsection{KL-UCB}

This algorithm is another variant of the upper confidence bound (UCB) class specifically designed for bounded rewards. In particular, it is known to be optimal for Bernoulli distributed rewards~\citep{garivier_kl-ucb_2011,cappeKullbackLeiblerUpper2013}.
For KL-UCB, $F_i$ is expressed as follows:

\begin{equation}\label{KLUCB}
\begin{split}
F_i = \mathrm{\max} \Bigl\{ &\theta \in \Theta:\ndrawt{i} \kull \left( \frac{\rewt{i}}{\ndrawt{i}}, \theta \right) \leq \ln(t) + c(\marmi{1},\marmi{2})\ln(\ln( t)) \Bigl\},
\end{split}
\end{equation}

where $\Theta$ denotes the definition interval of the posterior distribution. By testing various $c$ values for uniform priors, we end up with  $c(\marmi{1},\marmi{2}) = 0.00001$ ($c=0$ for the 50-armed Gaussian setting and $c=10^{-6}$ for the 2-armed Bernoulli setting). Of note, the maximum is found using a dichotomy method using a precision of $10^{-5}$ and a maximum number of iterations of 50.

For KL-UCB++ \citep{menardMinimaxAsymptoticallyOptimal2017}, the function $F_i$ is expressed as follows 

\begin{equation}\label{KLUCBplusplus}
\begin{split}
F_i = \mathrm{\max} \Bigl\{ &\theta \in \Theta: \ndrawt{i} \kull \left( \frac{\rewt{i}}{\ndrawt{i}}, \theta \right) \leq \ln_+ \left( \frac{T}{K \ndrawt{i}(t)} \ln_+^2 \left( \frac{T}{K  \ndrawt{i}(t)} \right) +
 1 \right)  \Bigl\},
\end{split}
\end{equation}

where $\ln_+(x) = \max(\ln(x),0)$, and $T$ is the stopping time of the bandit game. Therefore, KLUCB++ is not an anytime algorithm, but it still underperforms when compared to AIM and Thompson sampling. 


\subsubsection{Thompson sampling}

At each step, Thompson sampling~\citep{thompson1933likelihood,kaufmann2012thompson,kaufmannThompsonSamplingAsymptotically2012a} selects an arm at random, based on the posterior probability that is maximizes the expected reward. In practice, it draws $K$ random values according to each arm's mean posterior distribution and selects the arm with the highest sampled value as:
\begin{equation}\label{Thomposn}
\Armt = \underset{i=1..K}{\mathrm{argmax}} \bigg( Z_i \left(\meana{i}(t), \ndrawt{i}\right)  \bigg),
\end{equation}
where $Z_i(t)$ is drawn according to the posterior distribution of the $i$th arm's mean. Here, we used a uniform prior on $[0,1]$ for Bernoulli rewards and a uniform prior on $\mathbb{R}$ for Gaussian rewards to provide a direct comparison with \algoname[.]

Finally, for Thompson sampling plus~\citep{jinFiniteTimeRegretThompson2022}, denoted TS+, each sampled value for the comparison is drawn according to $ Z_i \left(\meana{i}(t), \ndrawt{i}\right)$ with a probability $1/K$ or taken equal to $\meana{i}(t)$, otherwise.

\subsubsection{MED}

At each step, the minimal empirical divergence (MED) algorithm~\citep{hondaAsymptoticallyOptimalPolicy2011a}, selects an arm at random, based on a tailored distribution building on the \kulleib distance to the better empirical arm. In practice, the arm $\Armt = i$ is be drawn with a probability: 
\begin{equation}
p_i = \frac{\exp \left[ - \ndrawt{i} \kull \left( \frac{\rewt{i}}{\ndrawt{i}},\meanab{\maxa} \right) \right]}{\sum_{j=0}^K \exp \left[ - \ndrawt{j} \kull \left( \frac{\rewt{j}}{\ndrawt{j}},\meanab{\maxa} \right) \right]  } .
\end{equation}

\subsection{Additional experiments}\label{app:additionalexperiments}

\subsubsection{Approximate information maximization for Gaussian rewards and close arms }\label{app:fixedmeanexpe}

For completeness, we provide in \cref{fig:closemean} below regret performances in which the arms' mean values are close ($ \Delta \mu = 0.01$), and are thus difficult to distinguish, for Gaussian reward distributions. Here, \algoname shows state-of-the-art performance comparable to Thompson sampling, even outperforming it at longer times. 

\begin{figure}[h!]
\centering
\includegraphics[scale=0.5]{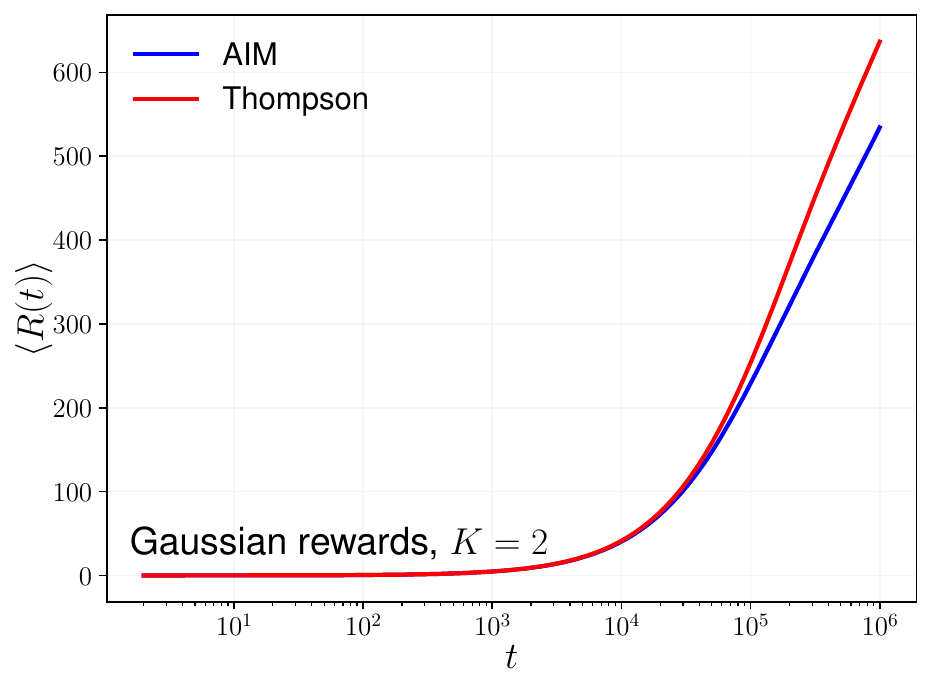}
\caption{Temporal evolution of the regret for 2-armed bandit with Gaussian rewards ($\sigma=1$) for close mean parameters. In blue \algonamese, in red Thompson sampling. 
Arm mean reward values are fixed with $\mu_1 = 0.8$ and $\mu_2 = 0.79$, the regret is obtained by averaging over $10^5$ realizations. }
\label{fig:closemean}
\end{figure}

\subsubsection{Approximate information maximization for Bernoulli rewards and close arms }\label{app:fixedmeanexpe2}

For completeness, we provide in \cref{fig:closemeanber} below regret performances in which the arms mean value are close ($ \Delta \mu = 0.01$) for Bernoulli reward distributions. As for Gaussian rewards, \algoname shows state-of-the-art performance comparable to Thompson sampling even when arms mean rewards are difficult to distinguish.

\begin{figure}[h!]
\centering
\includegraphics[scale=0.5]{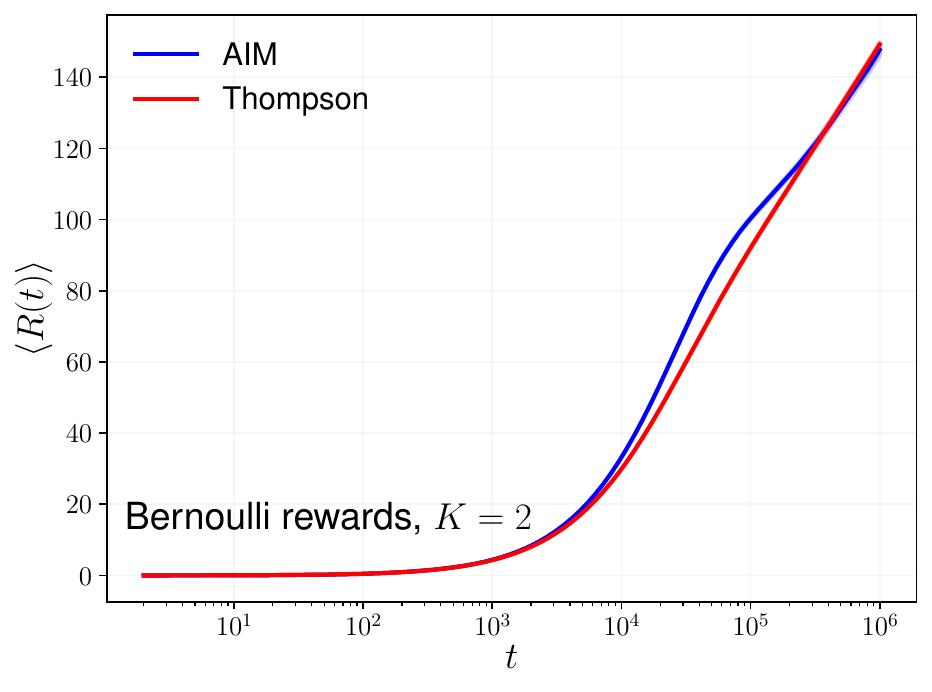}
\caption{Temporal evolution of the regret for 2-armed bandit with Bernoulli rewards for close mean parameters. In blue \algonamese, in red Thompson sampling. Arm mean reward values are fixed with $\mu_1 = 0.8$ and $\mu_2 = 0.79$, the regret is obtained by averaging over $10^5$ realizations. Confidence intervals shows the standard deviation. }
\label{fig:closemeanber}
\end{figure}

\subsubsection{Investigating approximate information maximization for large simulation volumes }\label{app:moredata}

To refine the numerical investigation of AIM's regret performance, we replicate the experiments of the main text using a larger volume of simulations (but on shorter timescales). For the 2-armed bandit with Gaussian and Bernoulli rewards, the regret performance under a uniform prior is averaged over more than $10^5$ runs. Similarly, for the 50-armed bandit with Bernoulli rewards, the regret is averaged over $4\times10^4$.
This leads to the results shown in \crefrange{fig:1moredata}{fig:3moredata}, confirming the results of \crefrange{fig:1}{fig:3}.

\begin{figure}[htbp]
\centering
\includegraphics[scale=0.5]{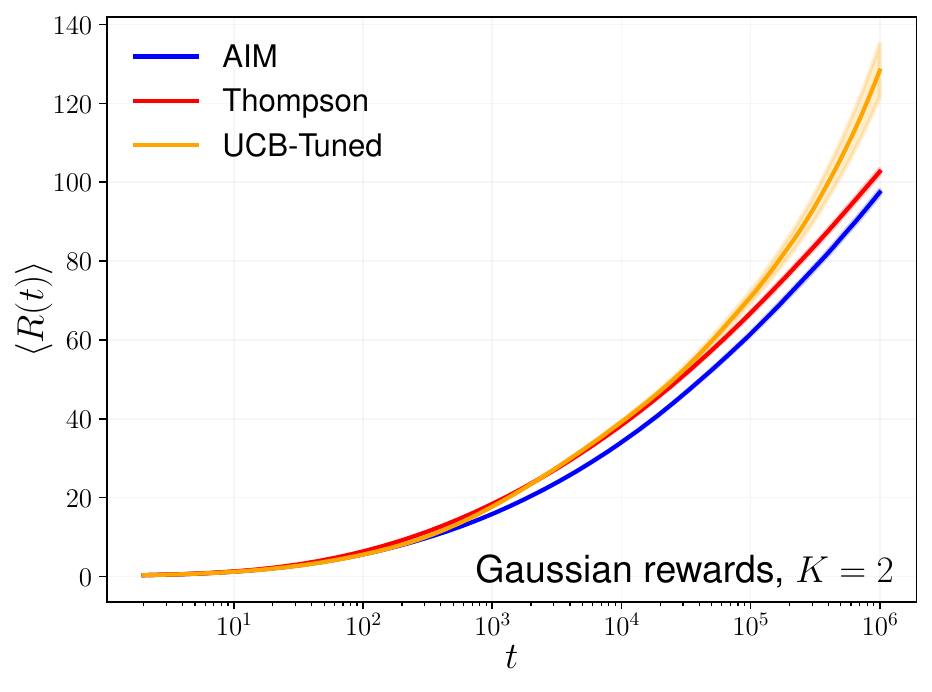}
\caption{Evolution of the Bayesian regret for 2-armed bandit with Gaussian rewards under a uniform mean prior. The regret is averaged over more than $10^5$ runs. Confidence intervals shows the standard deviation. Confidence intervals show the standard deviation.}
\label{fig:1moredata}
\end{figure}


\begin{figure}[htbp]
\centering
\includegraphics[scale=0.5]{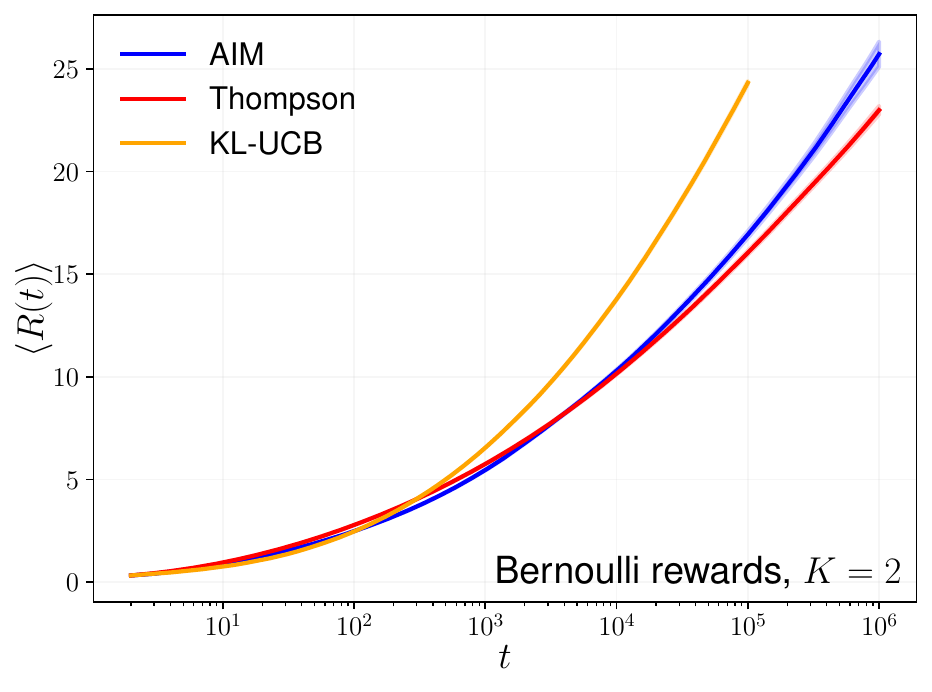}
\caption{Evolution of the Bayesian regret for 2-armed bandit with Bernoulli rewards under a uniform mean prior. The regret is averaged over more than $10^5$ runs. Confidence intervals show the standard deviation.}
\label{fig:2moredata}
\end{figure}

\begin{figure}[htbp]
\centering
\includegraphics[scale=0.5]{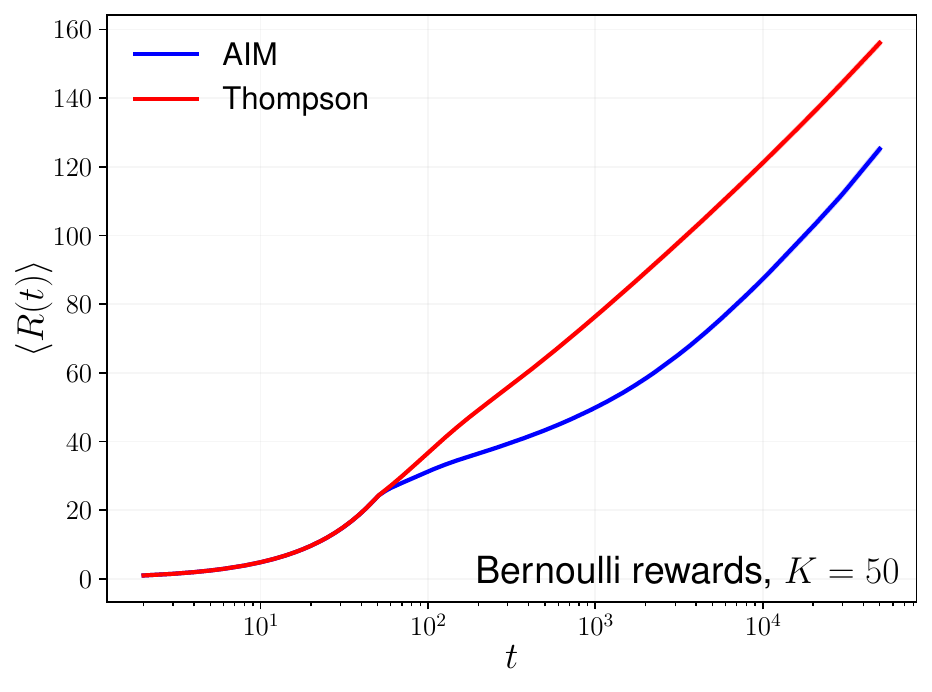}
\caption{Evolution of the Bayesian regret for 50-armed bandit with Bernoulli rewards under a uniform mean prior. The regret is averaged over $4\times10^4$ runs. Confidence intervals show the standard deviation.}
\label{fig:3moredata}
\end{figure}

\end{document}